\documentclass{article}

\usepackage{microtype}
\usepackage{graphicx}
\usepackage{adjustbox}
\usepackage{multirow}
\usepackage{subfigure}
\usepackage{booktabs} 
\usepackage{nicefrac}  
\usepackage{float}

\usepackage{amsmath,amsfonts,bm,amsthm}

\newcommand{\norm}[1]{\left\Vert#1\right\Vert}

\newcommand{\abs}[1]{\left\vert#1\right\vert}

\newcommand{\set}[1]{\left\{#1\right\}}
\newcommand{\parr}[1]{\left (#1\right )}

\newcommand{\Real}{\mathbb R}

\newcommand{\eps}{\varepsilon}
\newcommand{\veps}{\bm{\varepsilon}}
\newcommand{\too}{\rightarrow}

\newcommand{\diag}{\textrm{diag}} 
\newcommand{\dist}{\textrm{d}} 

\newcommand{\zero}{\mathbf{0}}

\newcommand{\eg}{{e.g.}}
\newcommand{\ie}{{i.e.}}

\makeatletter
\newtheorem*{rep@theorem}{\rep@title}
\newcommand{\newreptheorem}[2]{%
\newenvironment{rep#1}[1]{%
 \def\rep@title{#2 \ref{##1}}%
 \begin{rep@theorem}}%
 {\end{rep@theorem}}}
\makeatother

\newtheorem{theorem}{Theorem}
\newreptheorem{theorem}{Theorem}
\newtheorem{lemma}{Lemma}
\newreptheorem{lemma}{Lemma}

\usepackage{amsmath,amsfonts,bm}

\def\eqref#1{equation~\ref{#1}}
\def\Eqref#1{Equation~\ref{#1}}

\def\1{\bm{1}}

\def\eps{{\epsilon}}

\def\vzero{{\bm{0}}}

\def\vb{{\bm{b}}}

\def\ve{{\bm{e}}}

\def\vn{{\bm{n}}}

\def\vq{{\bm{q}}}

\def\vu{{\bm{u}}}
\def\vv{{\bm{v}}}
\def\vw{{\bm{w}}}
\def\vx{{\bm{x}}}
\def\vy{{\bm{y}}}
\def\vz{{\bm{z}}}

\def\mD{{\bm{D}}}

\def\mU{{\bm{U}}}

\def\mW{{\bm{W}}}

\DeclareMathAlphabet{\mathsfit}{\encodingdefault}{\sfdefault}{m}{sl}
\SetMathAlphabet{\mathsfit}{bold}{\encodingdefault}{\sfdefault}{bx}{n}

\def\gD{{\mathcal{D}}}

\def\gH{{\mathcal{H}}}

\def\gM{{\mathcal{M}}}
\def\gN{{\mathcal{N}}}

\def\gR{{\mathcal{R}}}
\def\gS{{\mathcal{S}}}

\def\gX{{\mathcal{X}}}
\def\gY{{\mathcal{Y}}}

\newcommand{\E}{\mathbb{E}}

\DeclareMathOperator*{\argmin}{arg\,min}

\usepackage{hyperref}

\usepackage{wrapfig}

\usepackage[accepted]{icml2020}

\icmltitlerunning{Implicit Geometric Regularization for Learning Shapes}

\begin{document}

\twocolumn[
\icmltitle{Implicit Geometric Regularization for Learning Shapes}



\icmlsetsymbol{equal}{*}

\begin{icmlauthorlist}
\icmlauthor{Amos Gropp}{wis}
\icmlauthor{Lior Yariv}{wis}
\icmlauthor{Niv Haim}{wis}
\icmlauthor{Matan Atzmon}{wis}
\icmlauthor{Yaron Lipman}{wis}
\end{icmlauthorlist}

\icmlaffiliation{wis}{Department of Computer Science \& Applied Mathematics, Weizmann Institute of Science, Rehovot, Israel}

\icmlcorrespondingauthor{Amos Gropp}{amos.gropp@weizmann.ac.il}

\icmlkeywords{Machine Learning, ICML}

\vskip 0.3in
]



\printAffiliationsAndNotice{}  

\begin{abstract}
Representing shapes as level sets of neural networks has been recently proved to be useful for different shape analysis and reconstruction tasks. So far, such representations were computed using either: (i) pre-computed implicit shape representations; or (ii) loss functions explicitly defined over the neural level sets. 

In this paper we offer a new paradigm for computing high fidelity implicit neural representations directly from raw data (\ie, point clouds, with or without normal information). We observe that a rather simple loss function, encouraging the neural network to vanish on the input point cloud and to have a unit norm gradient, possesses an implicit geometric regularization property that favors smooth and natural zero level set surfaces, avoiding bad zero-loss solutions. 

We provide a theoretical analysis of this property for the linear case, and show that, in practice, our method leads to state of the art implicit neural representations with higher level-of-details and fidelity compared to previous methods. 
\end{abstract}

\begin{figure}[t]
    \begin{tabular}{ccc}
      \hspace{-8pt} \includegraphics[width=0.31\columnwidth]{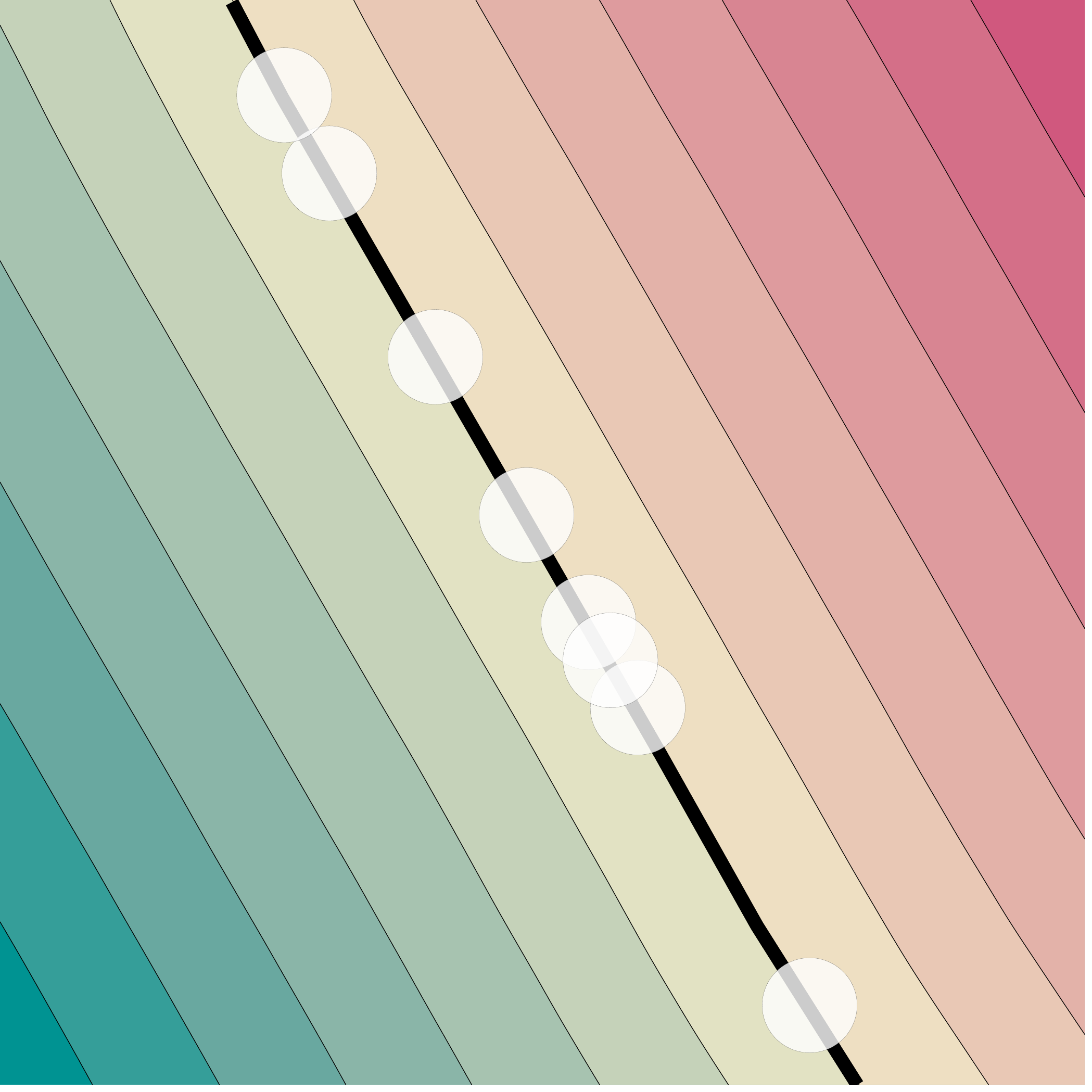} & 
      \hspace{-8pt} \includegraphics[width=0.31\columnwidth]{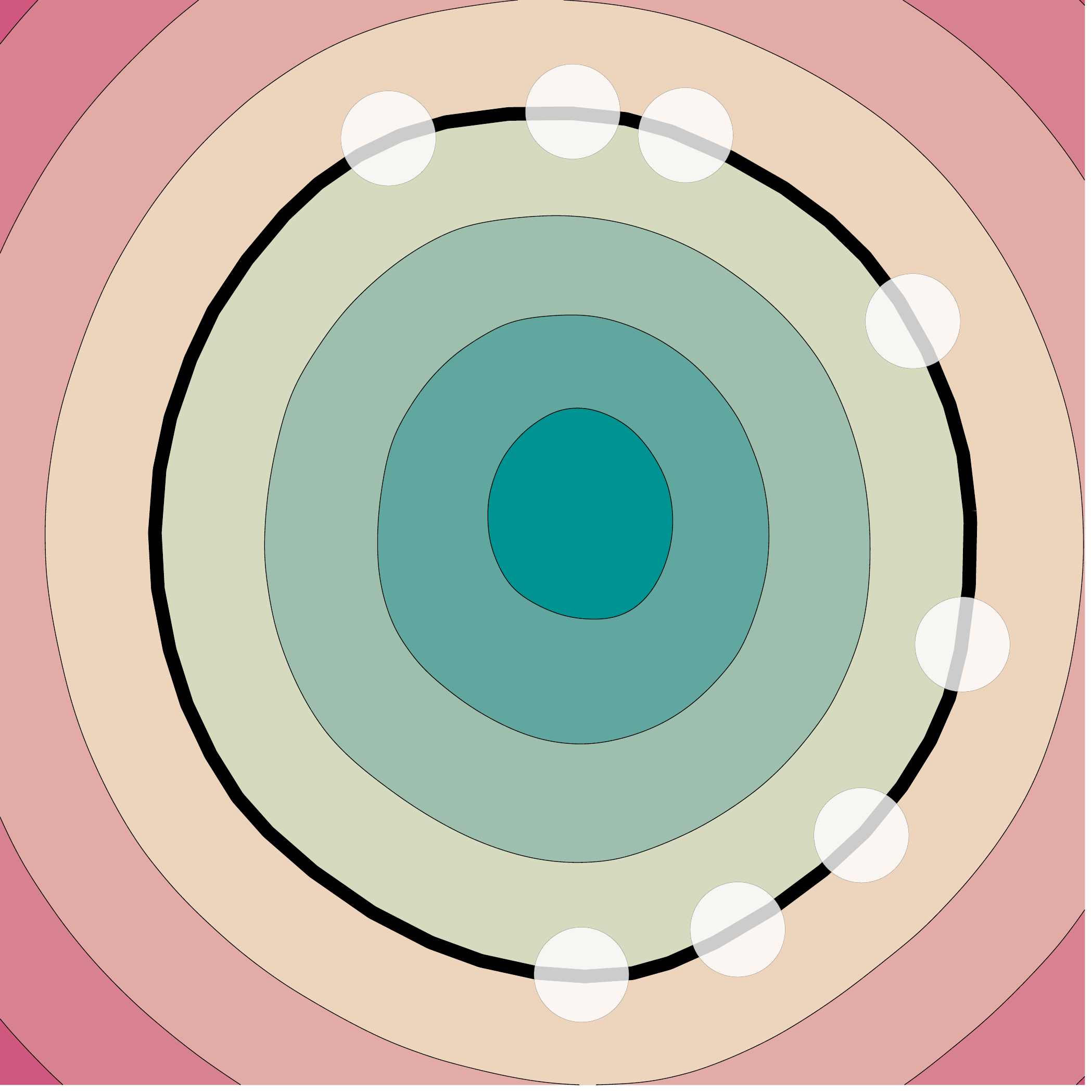} & 
      \hspace{-8pt} \includegraphics[width=0.31\columnwidth]{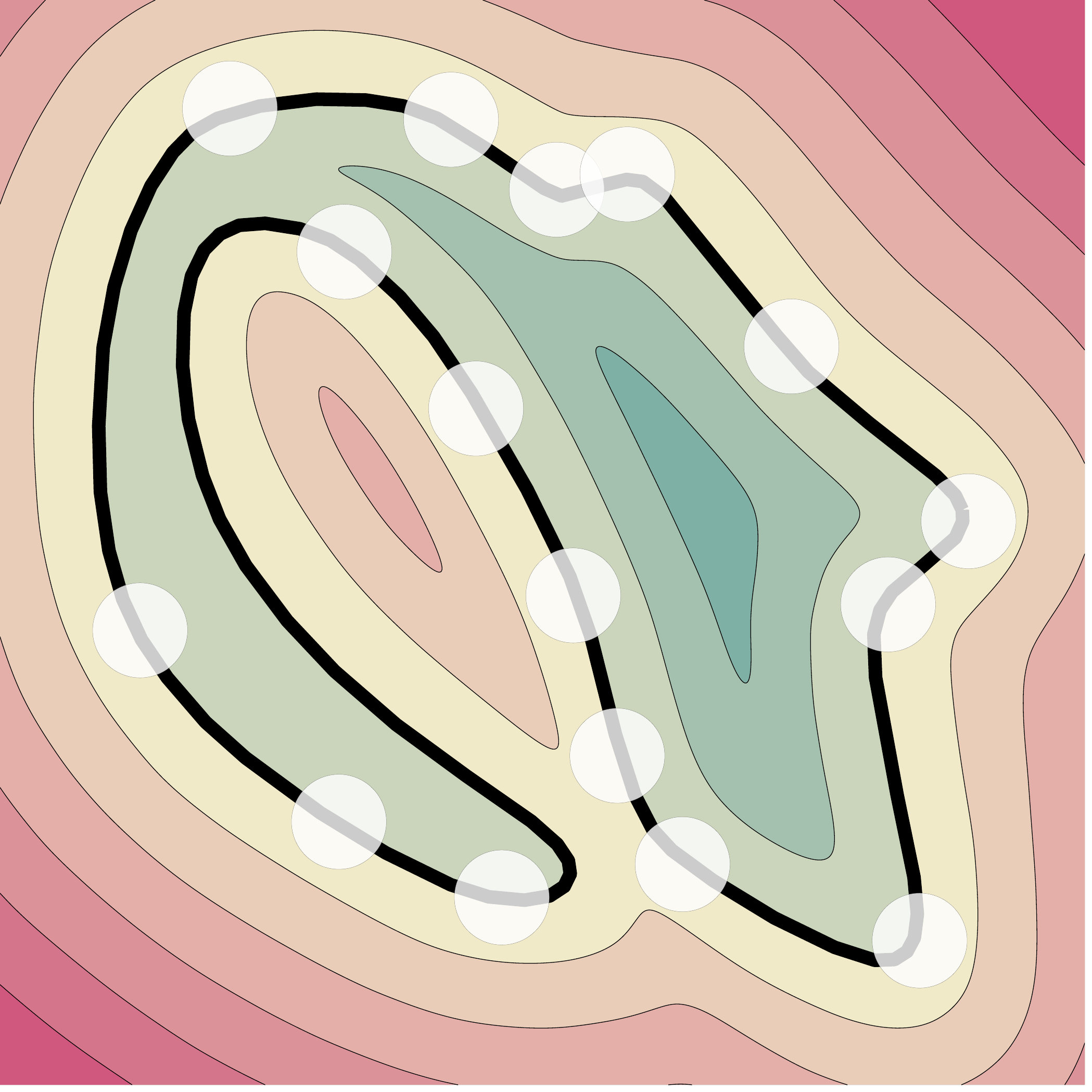} \\
      \hspace{-8pt} \includegraphics[width=0.31\columnwidth]{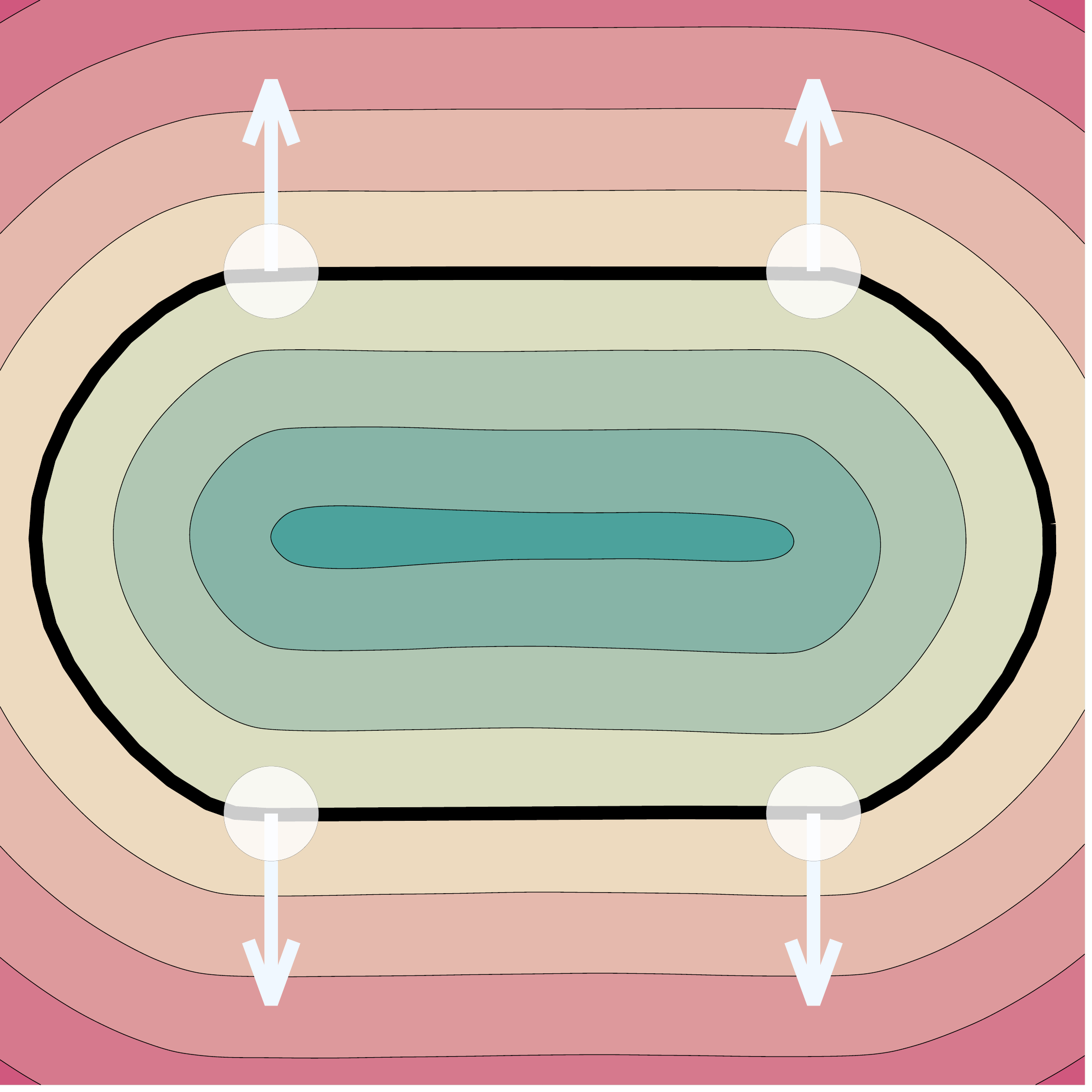} & 
      \hspace{-8pt} \includegraphics[width=0.31\columnwidth]{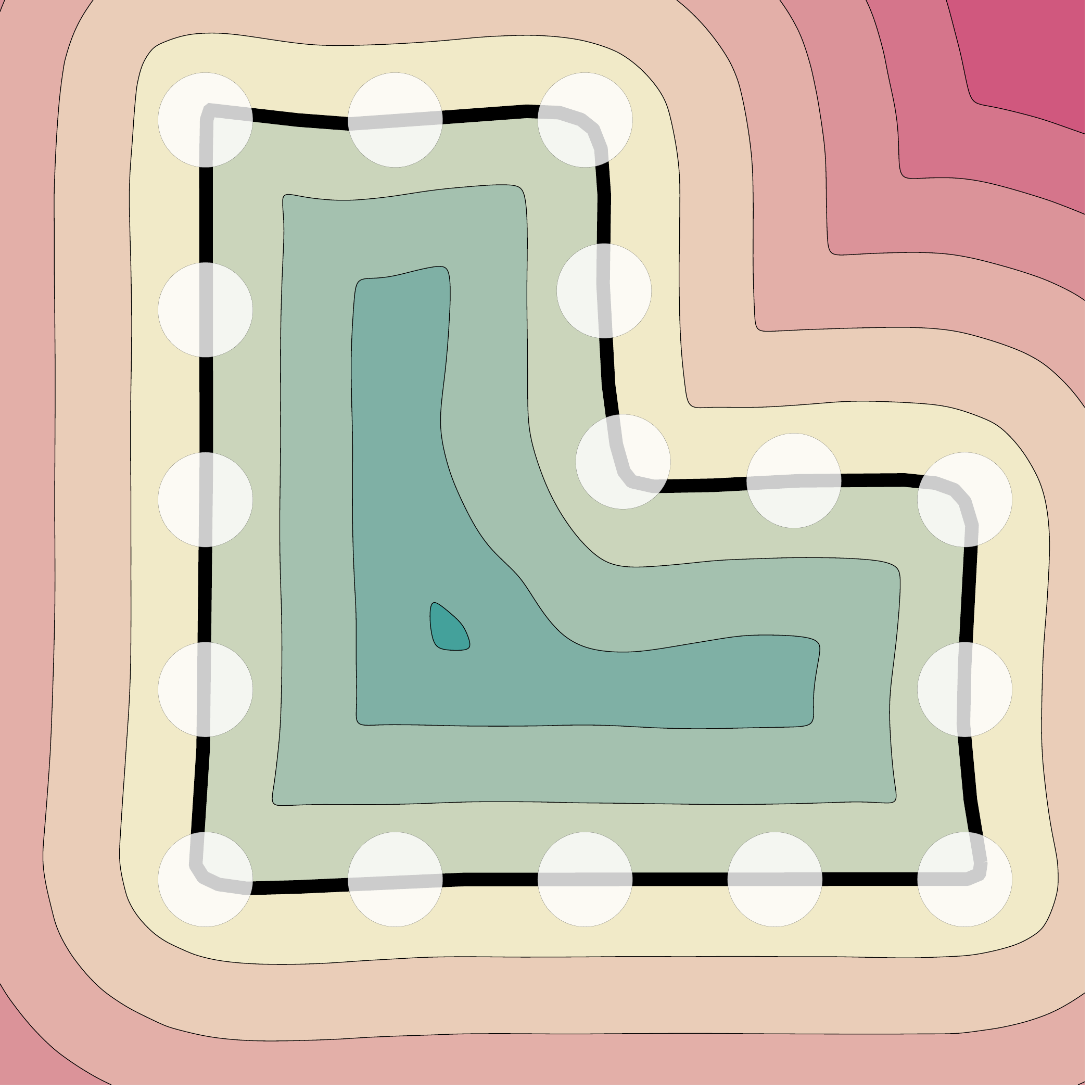} & 
      \hspace{-8pt} \includegraphics[width=0.31\columnwidth]{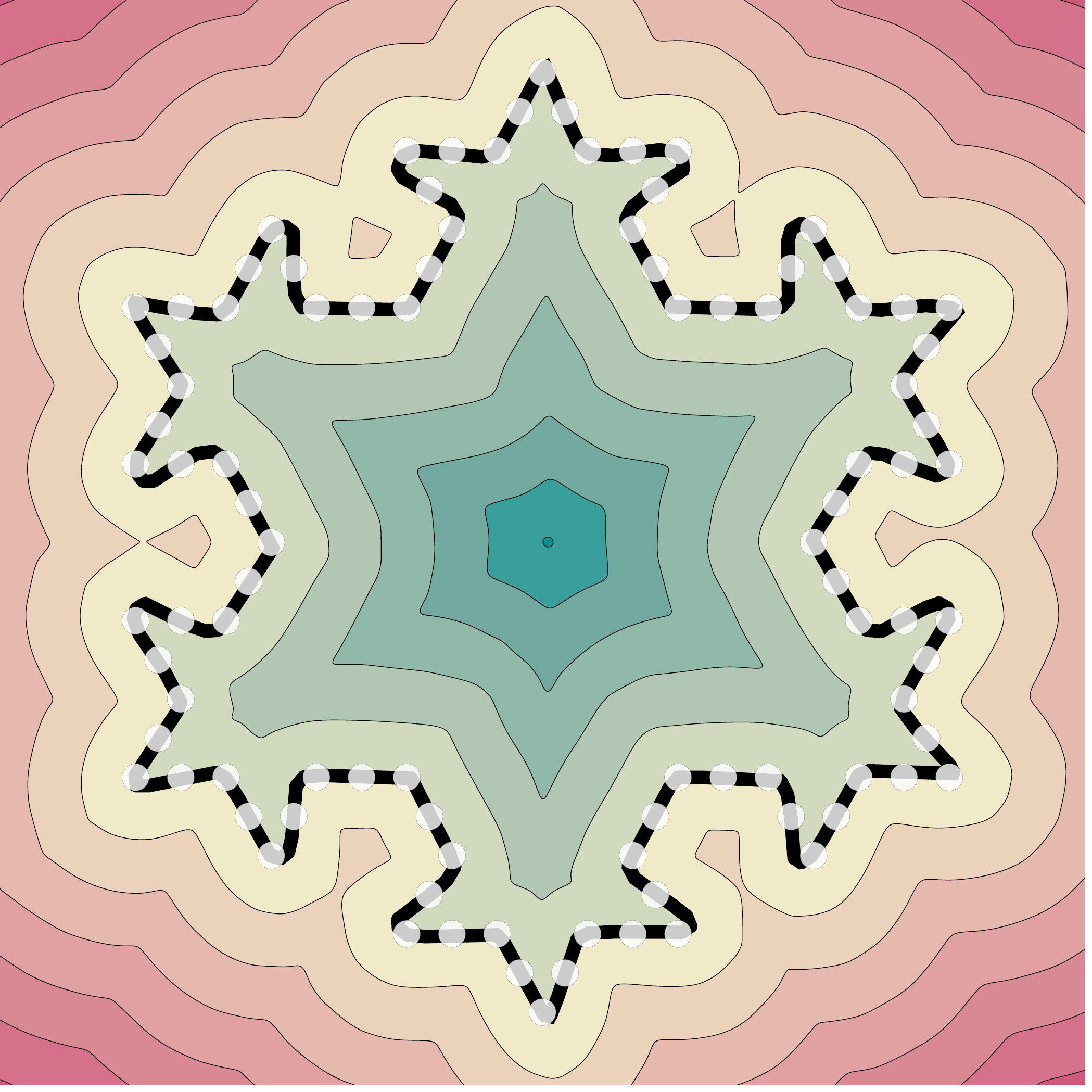} \vspace{-5pt} \\
\end{tabular}
    \caption{Learning curves from 2D point clouds (white disks) using our method; black lines depict the zero level sets of the trained neural networks, $\gM$. The implicit geometric regularization drives the optimization to reach plausible explanation of the data. }
    
    \label{fig:2d}
\vspace{-10pt}
\end{figure}

\section{Introduction}
Recently, level sets of neural networks have been used to represent 3D shapes \cite{park2019deepsdf,atzmon2019controlling,chen2019learning,mescheder2019occupancy}, \ie,
\begin{equation}\label{e:gM}    
\gM = \set{\vx\in \Real^3 \ \vert \ f(\vx;\theta)=0}, 
\end{equation}
where $f:\Real^{3}\times \Real^m \too \Real$ is a multilayer perceptron (MLP); we call such representations \emph{implicit neural representations}. Compared to the more traditional way of representing surfaces via implicit functions defined on volumetric grids \cite{wu2016learning,choy20163d,dai2017shape,stutz2018learning}, neural implicit representations have the benefit of relating the degrees of freedom of the network (\ie, parameters) directly to the shape rather than to the fixed discretization of the ambient 3D space. 
So far, most previous works using implicit neural representations computed $f$ with 3D supervision; that is, by comparing $f$ to a known (or pre-computed) implicit representation of some shape. \cite{park2019deepsdf} use a regression loss to approximate a pre-computed signed distance functions of shapes; \cite{chen2019learning,mescheder2019occupancy} use classification loss with pre-computed occupancy function. 

In this work we are interested in working directly with raw data: Given an input point cloud $\gX=\set{\vx_i}_{i\in I} \subset \Real^3$, with or without normal data, $\gN=\set{\vn_i}_{i\in I}\subset \Real^3$, our goal is to compute $\theta$ such that $f(\vx;\theta)$ is approximately the signed distance function to a plausible surface $\gM$ defined by the point data $\gX$ and normals $\gN$.

Some previous works are constructing implicit neural representations from raw data. In \cite{atzmon2019controlling} no 3D supervision is required and the loss is formulated directly on the zero level set $\gM$; iterative sampling of the zero level set is required for formulating the loss. In a more recent work, \cite{atzmon2020sal} use unsigned regression to introduce good local minima that produces useful implicit neural representations, with no 3D supervision and no zero level set sampling. Both of these works, however, explicitly enforce some regularization on the zero level set. Another possible solution is to compute, as a pre-training stage, some implicit representation of the surface using a classical surface reconstruction technique and use one of the previously mentioned methods to construct the neural implicit representation. This approach has two drawbacks: First, finding implicit surface representation from raw data is a notoriously hard \cite{berger2017survey}; second, decoupling the reconstruction from the learning stage would hinder \emph{collective} learning and reconstruction of shapes. For example, information from one shape will not be used to improve reconstruction of a different, yet a similar shape; nor consistent reconstructions will be produced. 
\begin{figure}
    \centering
    \begin{tabular}{c}
     \includegraphics[width=1.0\columnwidth]{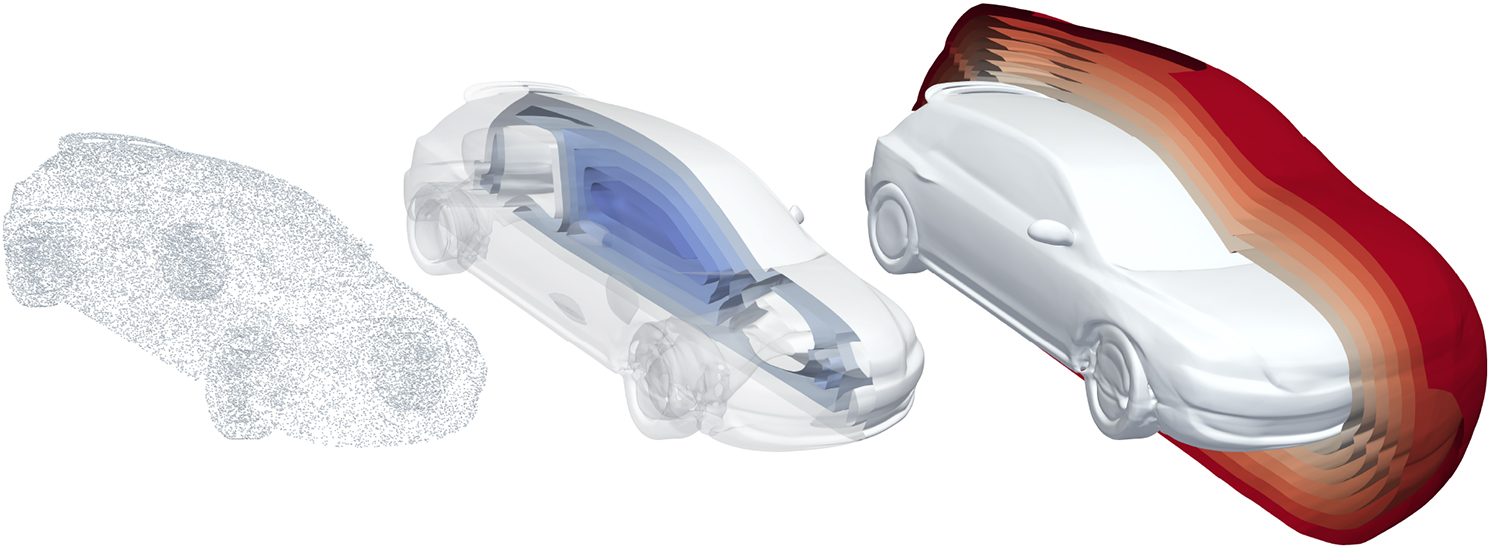}\vspace{-6pt}
\end{tabular}
    \caption{The level sets of an MLP (middle and right) trained with our method on an input point cloud (left); positive level sets are in red; negative are in blue; the zero level set, representing the approximated surface is in white. }
    \label{fig:car_levelset}
    \vspace{-13pt}
\end{figure}

The goal of this paper is to introduce a novel approach for learning neural shape representations directly from raw data using \emph{implicit geometric regularization}. We show that state of the art implicit neural representations can be achieved without 3D supervision and/or a direct loss on the zero level set $\gM$. As it turns out, stochastic gradient optimization of a simple loss that fits an MLP $f$ to a point cloud data $\gX$, with or without normal data $\gN$, while encouraging unit norm gradients $\nabla_\vx f$, consistently reaches good local minima, favoring smooth, yet high fidelity, zero level set surfaces $\gM$ approximating the input data $\gX$ and $\gN$. Figure \ref{fig:2d} shows several implicit neural representation of 2D data computed using our method; note that although there is an infinite number of solutions with neural level sets interpolating the input data, the optimization reaches solutions that provide natural and intuitive reconstructions. Figure \ref{fig:car_levelset} shows the level sets of an MLP trained with our method from a point cloud shown on the left.  

The preferable local minima found by the optimization procedure could be seen as a geometric version of the well known \emph{implicit regularization} phenomenon in neural network optimization \cite{neyshabur2014search,neyshabur2017implicit,soudry2018implicit}. Another, yet different, treatment of geometric implicit regularization was discussed in \cite{williams2019deep} where reducing an entropic regularization in the loss still maintains consistent and smooth neural chart representation. 

Although we do not provide a full theory supporting the implicit geometric regularization phenomenon, we analyze the linear case, which is already non-trivial due to the non-convex unit gradient norm term. We prove that if the point cloud $\gX$ is sampled (with small deviations) from a hyperplane $\gH$ and the initial parameters of the linear model are randomized, then, with probability one, gradient descent converges to the (approximated) signed distance function of the hyperplane $\gH$, avoiding bad critical solutions. We call this property \emph{plane reproduction} and advocate it as a useful geometric manifestation of implicit regularization in neural networks.

In practice, we perform experiments with our method, building implicit neural representations from point clouds in 3D and learning collections of shapes directly from raw data. Our method produces state of the art surface approximations, showing significantly more detail and higher fidelity compared to alternative techniques. Our code is available at \url{https://github.com/amosgropp/IGR}.

In summary, our paper's main contribution is two-fold: 
\begin{itemize}
    \item Suggesting a new paradigm, building on implicit geometric regularization, for computing high fidelity implicit neural representations, directly from raw data.
    \item Providing a theoretical analysis for the linear case showing gradient descent of the suggested loss function avoids bad critical solutions.
\end{itemize}

\section{Method}
Given an input point cloud $\gX=\set{\vx_i}_{i\in I} \subset \Real^3$, with or without normal data, $\gN=\set{\vn_i}_{i\in I}\subset \Real^3$, our goal is to compute parameters $\theta$ of an MLP $f(\vx;\theta)$, where $f:\Real^3\times \Real^m \too \Real$, so that it approximates a signed distance function to a plausible surface $\gM$ defined by the point cloud $\gX$ and normals $\gN$. 

We consider a loss of the form 
\begin{equation}\label{e:loss}
    \ell(\theta) = \ell_{\gX}(\theta) + \lambda\E_\vx\big (\norm{\nabla_\vx f(\vx;\theta)}-1\big )^2,
\end{equation}
where $\lambda>0$ is a parameter, $\norm{\cdot}=\norm{\cdot}_2$ is the euclidean $2$-norm, and  $$\ell_{\gX}(\theta)=\frac{1}{|I|}\sum_{i\in I}\big ( \abs{ f(\vx_i;\theta) } + \tau \norm{\nabla_\vx f(\vx_i;\theta) - \vn_i} \big )$$ encourages $f$ to vanish on $\gX$ and, if normal data exists (\ie, $\tau=1$), that $\nabla_\vx f$ is close to the supplied normals $\gN$. 

The second term in \eqref{e:loss} is called the \emph{Eikonal term} and encourages the gradients $\nabla_\vx f$ to be of unit 2-norm. The expectation is taken with respect to some probability distribution $\vx\sim \gD$ in $\Real^3$. 

\pagebreak
The motivation for the Eikonal term stems from the Eikonal partial differential equation: A solution $f$ (in the sense of \cite{crandall1983viscosity}) to the Eikonal equation, \ie, 
\begin{equation}\label{e:eikonal}
    \norm{\nabla_\vx f(\vx)} = 1,     
\end{equation}
where $f$ vanishes on $\gX$, with gradients $\gN$, will be a signed distance function and a global minimum of the loss in \eqref{e:loss}. Note however, that for point boundary data $\gX,\gN$ the solution to \eqref{e:eikonal} is not unique.

\paragraph{Implicit geometrical regularization.}
When optimizing the loss in \eqref{e:loss}, two questions immediately emerge: First, why a critical point $\theta_*$ that is found by the optimization algorithm leads $f(\vx;\theta_*)$ to be a signed distance function? Usually, adding a quadratic penalty with a finite weight is not guaranteed to provide \emph{feasible} critical solutions \cite{nocedal2006numerical}, \ie, solutions that satisfy the desired constraint (in our case, unit length gradients). Second, even if the critical solution found is a signed distance function, why would it be a plausible one? There is an infinite number of signed distance functions vanishing on arbitrary discrete sets of points $\gX$ with arbitrary normal directions $\gN$. 

Remarkably, optimizing \eqref{e:loss} using stochastic gradient descent (or a variant thereof) results in solutions that are close to a signed distance function with a smooth and surprisingly plausible zero level set. For example, Figure \ref{fig:2d} depicts the result of optimizing \eqref{e:loss} in the planar case ($d=2$) for different input point clouds $\gX$, with and without normal data $\gN$; the zero level sets of the optimized MLP are shown in black.  Note that the optimized MLP is close to a signed distance function as can be inspected from the equidistant level sets. 

\begin{wrapfigure}[7]{r}{0.3\columnwidth}
     \centering
     \vspace{-8pt}
     \hspace{-10pt}
     \includegraphics[width=0.31\columnwidth]{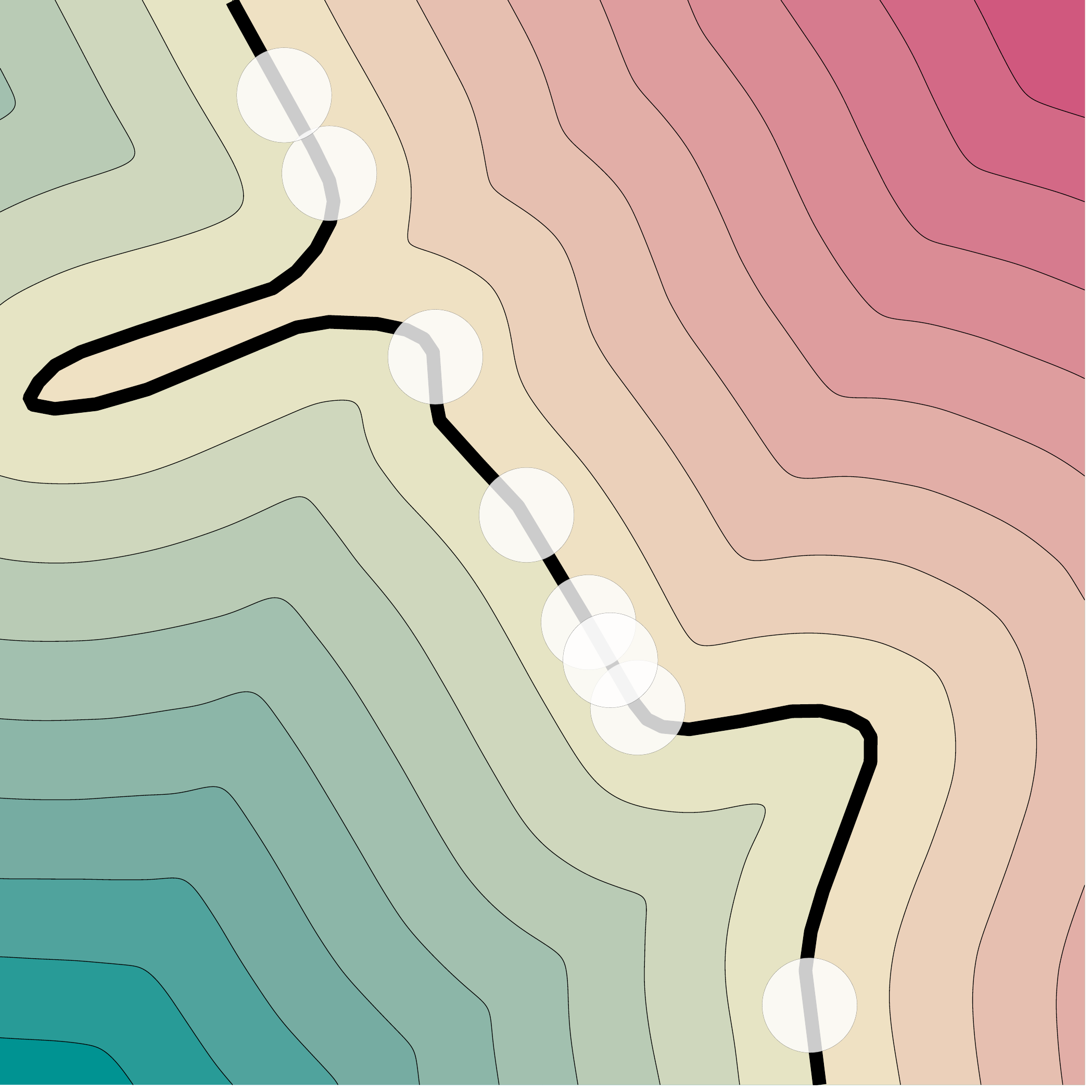}
\end{wrapfigure}
The inset shows an alternative signed distance function that would achieve zero loss in \eqref{e:loss} for the top-left example in Figure \ref{fig:2d}, avoided by the optimization algorithm that chooses to reconstruct a straight line in this case. 
This is a consequence of the \emph{plane reproduction} property. In Section \ref{s:analysis} we provide a theoretical analysis of the plane reproduction property for the linear model case $f(\vx;\vw)=\vw^T\vx$ and prove that if $\gX$ is sampled with small deviations from a hyperplane $\gH$, then gradient descent provably avoids bad critical solutions and converges with probability one to the approximated signed distance function to $\gH$.

\paragraph{Computing gradients.}

Incorporating the gradients $\nabla_\vx f$ in the loss above could be done using numerical estimates of the gradient. A better approach is the following: every layer of the MLP $f$ has the form  $\vy^{\ell+1}=\sigma(\mW\vy^{\ell}+\vb)$, where $\sigma:\Real\too\Real$ is a non-linear differentiable activation (we use Softplus) applied entrywise, and $\mW$, $\vb$ are the layer's learnable parameters. Hence, by the chain-rule the gradients satisfy
\begin{equation}\label{e:grads}
  \nabla_\vx \vy^{\ell+1}=\diag\parr{\sigma'\parr{\mW\vy^{\ell}+\vb}}\mW\nabla_\vx \vy^{\ell},  
\end{equation}
\begin{wrapfigure}[6]{r}{0.45\columnwidth}
     \centering
     \vspace{-10pt}\hspace{-8pt}
     \includegraphics[width=0.42\columnwidth]{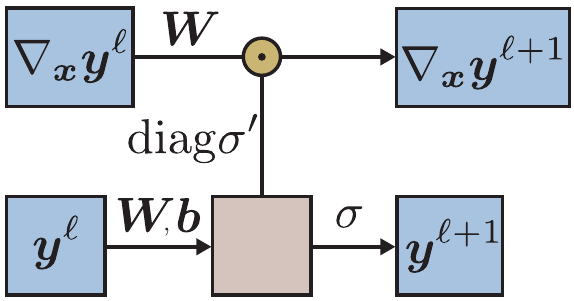}\vspace{-10pt}
     \caption{$(f,\nabla_\vx f)$ network.}\label{fig:grad}
\end{wrapfigure}
where $\diag(\vz)$ is arranging its input vector $\vz\in\Real^k$ on the diagonal of a square matrix $\Real^{k\times k}$ and $\sigma'$ is the derivative of $\sigma$. \Eqref{e:grads} shows that $\nabla_\vx f(\vx;\theta)$ can be constructed as a neural-network in conjunction with $f(\vx;\theta)$, see Figure \ref{fig:grad} for illustration of a single layer of a network computing both $f(\vx;\theta)$ and $\nabla_\vx f(\vx;\theta)$. In practice, implementing $\nabla_\vx f(\vx;\theta)$ using \emph{Automatic Differentiation} packages seems to be a simple alternative.

\section{Previous work and discussion}
\subsection{Deep learning of 3D shapes}
There are numerous deep learning based methods applied to 3D shapes. Here we review the main approaches, emphasizing the 3D data representation being used and discuss relations to our approach.
\paragraph{\emph{RGVF: regular grid volumetric function}.}
Maybe the most popular representation for 3D shapes is via a scalar function defined over a regular volumetric grid (RGVF); the shape is then defined as the zero level set of the function. One option is to use an indicator function \cite{choy20163d,girdhar2016learning,wu2016learning,yan2016perspective,tulsiani2017multi,yang20173d}. This is a natural generalization of images to 3D, thus enabling easy adaptation of the successful Convolutional Neural Networks (CNNs) architectures. \citet{tatarchenko2017octree} addressed the computation efficiency challenge stemming from the cubic grid size. In \cite{wu2016learning}, a variant of generative adversarial network \cite{goodfellow2014generative} is proposed for 3D shapes generation. More generally, researchers have suggested using other RGVFs \cite{dai2017shape,riegler2017octnetfusion,stutz2018learning,liao2018deep,michalkiewicz2019deep,jiang2019sdfdiff}. In \cite{dai2017shape}, the RGVF models a signed distance function to a shape, where an encoder-decoder network is trained for the task of shape completion. Similar RGVF representation is used in \cite{jiang2019sdfdiff} for the task of multi-view 3D reconstruction. However, they learn the signed distance function representation based on differentiable rendering technique, without requiring pre-training 3D supervision. Another implicit RGVF representation is used in \cite{michalkiewicz2019deep} for the task of image to shape prediction; they introduced a loss function inspired by the \emph{level set method} \cite{osher2004level} based surface reconstruction techniques \cite{zhao2000implicit, zhao2001fast}, operating directly on level sets of RGVF.

The RGVF has several shortcomings compared to the implicit neural representations in general and our approach in particular: (i) The implicit function is defined only at grid points, requiring an interpolation scheme to extend it to interior of cells; normal approximation would require divided differences; (ii) It requires cubic-size grid and is not data-dependent, \ie, it does not necessarily adhere to the specific geometry of the shapes one wishes to approximate.
(iii) A version of the Eikonal regularization term (second term in \eqref{e:loss}) was previously used with RGVF representations in \citet{michalkiewicz2019deep, jiang2019sdfdiff}. These works incorporated Eikonal regularization as a normalization term, in combination with other explicit reconstruction and/or regularization terms. The key point in our work is that the Eikonal term \emph{alone} can be used for (implicitly) regularizing the zero level set. Furthermore, it is not clear whether the implicit regularization property holds, and to what extent, in the fixed volumetric grid scenario. Lastly, in the RGVF setting 
the gradients are computed using finite differences or via some fixed basis function.



\paragraph{Neural parametric surfaces.} Surfaces can also be described explicitly as a collection of charts (in an atlas), where each chart  $f:\Real^2\too \Real^3$ is a local parametrization. However, finding a consistent atlas covering of the surface could be challenging. \cite{groueix2018,deprelle2019learning} suggested modeling charts as MLPs; \cite{williams2019deep} focused on an individual surface reconstruction, introducing a method to improve atlas consistency by minimizing disagreement between neighboring charts. Some works have considered global surface parametrizations \cite{sinha2016deep,sinha2017surfnet,maron2017convolutional}. Global parametrizations produce consistent coverings, however at the cost of introducing parametrizations with high distortion. 

The benefit of parametric representations over implicit neural representations is that the shape can be easily sampled; the main shortcoming is that it is very challenging to produce a set of perfectly overlapping charts, a property that holds by construction for implicit representations.

\paragraph{Hybrid representations.} \citet{deng2019cvxnets,chen2019bsp,williams2019voronoinet} suggested representations based on the fact that every solid can be decomposed into a union of convex sets. As every convex set can be represented either as an intersection of hyper-planes or a convex-hull of vertices, transforming between a shape explicit and implicit representation can be done  relatively easily. 

\subsection{Solving PDEs with neural networks}
Our proposed training objective \eqref{e:loss} can be interpreted as a quadratic penalty method for solving the \emph{Eikonal Equation} (\eqref{e:eikonal}). However, the Eikonal equation is a non-linear wave equation and requires boundary conditions of the form $$f(\vx)=0, \quad \vx\in\partial \Omega,$$ where $\Omega\subset \Real^3$ is a well-behaved open set with boundary, $\partial \Omega$. The viscous solution \cite{crandall1983viscosity,crandall1984some} to the Eikonal equation is unique in this case. Researchers are trying to utilize neural networks to solve differential equations \cite{yadav2015introduction}. Perhaps, more related to our work are \citet{sirignano2018dgm,raissi2017physics,raissi2017physicst} suggesting deep neural networks as a non-linear function space for approximating PDE solutions. Their training objective, similarly to ours, can be seen as a penalty version of the original PDE. 

The main difference from our setting is that in our case the boundary conditions of the Eikonal equation do not hold, as we use a \emph{discrete} set of points $\gX$. In particular, any well-behaved domain $\Omega$ that contains $\gX$ in its boundary, \ie, $\gX\subset \partial\Omega$ would form a valid initial condition to the Eikonal equation with $\partial\Omega$ as the zero level set of its solution. Therefore, from PDE theory point of view, the problem \eqref{e:loss} is trying to solve is ill-posed with infinite number of solutions. The main observation of this paper is that implicit geometry regularization in fact \emph{chooses a favorable solution} out of this solution space. 

\section{Analysis of the linear model and plane reproduction }\label{s:analysis}

In this section we provide some justification for using the loss in \eqref{e:loss} by analyzing the linear network case. That is, we consider a linear model $f(\vx;\vw)=\vw^T\vx$ where the loss in \eqref{e:loss} takes the form 
\begin{equation}\label{e:loss_linear}
\ell(\vw) = \sum_{i\in I} \parr{\vw^T \vx_i}^2 + \lambda \parr{\norm{\vw}^2-1}^2,
\end{equation}
where for simplicity we used squared error and removed the term involving normal data; we present the analysis in $\Real^d$ rather than $\Real^3$.  

We are concerned with the \emph{plane reproduction} property, namely, assuming our point cloud $\gX$ is sampled approximately from a plane $\gH$, then gradient descent of the loss in \eqref{e:loss_linear} converges to the approximate signed distance function to $\gH$. 

To this end, assume our point cloud data $\gX=\set{\vx_i}_{i\in I}$ satisfies $\vx_i=\vy_i+\veps_i$, where  $\vy_i$, span some $d-1$-dimension hyperplane $\gH\subset\Real^d
$ that contains the origin, and $\veps_i$ are some small deviations satisfying $\norm{\veps_i}<\eps$. We will show that: (i) For $\lambda > \frac{c\eps}{2}$, where $c$ is a constant depending on $\vy_i$, the loss in \eqref{e:loss_linear} has two global minima that correspond to the (approximated) signed distance functions to $\gH$ (note there are two signed distance functions to $\gH$ differing by a sign); the rest of its critical points are either saddle points or local maxima. (ii) Using the characterization of critical points and known properties of gradient descent \cite{ge2015escaping,lee2016gradient} we can prove that applying gradient descent 
\begin{equation}\label{e:gradient_descent}
    \vw^{k+1}=\vw^k-\alpha \nabla_\vw \ell(\vw^k),
\end{equation}
from a random initialization $\vw^0$ and sufficiently small step-size $\alpha>0$, will converge, with probability one,  
to one of the global minima, namely to the approximated signed distance function to $\gH$.

\paragraph{Change of coordinates.} We perform a change of coordinates: Let $\sum_{i\in I}\vx_i\vx_i^T=\mU\mD\mU^T$, $\mU=(\vu_1,\ldots,\vu_d)$, $D=\diag(\lambda_1,\ldots,\lambda_d)$ be a spectral decomposition, and $0\leq \lambda_1 < \lambda_{2}\leq \cdots\leq \lambda_d$. Using perturbation theory for eigenvalues and eigenvectors of hermitian matrices one proves:
\begin{lemma}\label{lem:perturb}
There exists constants $c,c'>0$ depending on $\set{\vy_i}_{i\in I}$ so that $\lambda_1 \leq c\eps$ and $\norm{\vu_1-\vn}\leq c'\eps$, where $\vn$ is a normal direction to $\gH$. 
\end{lemma}
\begin{proof}[Proof (Lemma \ref{lem:perturb}).]
Let $\sum_{i\in I} \vx_i\vx_i^T=\sum_{i\in I} \vy_i\vy_i^T + \vy_i\veps_i^T+\veps_i\vy_i^T + \veps_i\veps_i^T=\sum_{i\in I} \vy_i\vy_i^T + E$. Now use hermitian matrix eigenvalue perturbation theory, \eg, \cite{stewart1990matrix}, Section IV:4, and perturbation theory for simple eigenvectors, see \cite{stewart1990matrix} Section V:2.2-2.3, to conclude.   
\end{proof}

Then, performing the change of coordinates,  $\vq=U^T\vw$, in \eqref{e:loss_linear} leads to the diagonalized form
\begin{equation}\label{e:loss_linear_changed}
  \ell(\vq)=\vq^T D \vq + \lambda \parr{\norm{\vq}^2-1}^2,  
\end{equation}
where  $\norm{\vw}=\norm{\vq}$ due to the invariance of the euclidean norm to orthogonal transformations. The plane $\gH$ in the transformed coordinates is $\ve_1^\perp$, where $\ve_1\in\Real^d$ is the first standard basis vector. 
 
\paragraph{Classification of critical points.} Next we classify the critical points of our loss. The gradient of the loss in \eqref{e:loss_linear_changed} is
\begin{equation}\label{e:loss_linear_grad}
\nabla_\vq \ell(\vq)^T = 2\parr{D + 2\lambda(\norm{\vq}^2-1)I}\vq,
\end{equation}
where $I$ is the identity matrix. The Hessian takes the form
\begin{equation}\label{e:loss_linear_hess}
\nabla^2_\vq \ell(\vq) = 2D + 4\lambda\parr{\norm{\vq}^2-1}I + 8\lambda \vq\vq^T.
\end{equation}
We prove:
\begin{theorem}\label{thm:critical}
If $\lambda > \frac{\lambda_1}{2}$, then the loss in \eqref{e:loss_linear_changed} (equivalently, \eqref{e:loss_linear}) has at-least $3$ and at-most $2d+1$ critical points. Out of which, the following two correspond to the approximated signed distance functions to the plane $\ve_1^\perp$, and are global minima: $$\pm\vq=\pm\sqrt{1-\frac{\lambda_1}{2\lambda}}\ve_1.$$ The rest of the critical points are saddle points or local maxima. 
\end{theorem}

Before proving this theorem we draw some conclusions. The global minima in the original coordinate system are $\pm\vw=\pm{\scriptstyle\sqrt{1-\frac{\lambda_1}{2\lambda}}}\vu_1$ that correspond to the approximate signed distance function to $\gH$. Indeed, Lemma \ref{lem:perturb} implies that $\lambda_1/2\lambda\leq c\eps/2\lambda$ and $\norm{\vu_1-\vn}\leq c'\eps$. Therefore $$\norm{\vw-\vn}\leq \parr{\frac{c}{2\lambda}+c'}\eps,$$ where we used the triangle inequality and $\sqrt{1-s}\geq 1-s$ for $s\in[0,1]$. This shows that $\lambda>0$ should be chosen sufficiently large compared to $\eps$, the deviation of the data from planarity. In the general MLP case, one could consider this analysis locally noting that locally an MLP is approximately linear and the deviation from planarity is quantified locally by the \emph{curvature} of the surface, \eg, the two principle surface curvatures $\sigma_1,\sigma_2$ that represent the reciprocal radii of two osculating circles \cite{do2016differential}.

\begin{proof}[Proof (Theorem \ref{thm:critical}).]
First, let us find all critical points. Clearly, $\vq=\zero$ satisfies $\nabla_\vq \ell(\zero)=0$. Now if $\vq\ne \zero$ then the only way \eqref{e:loss_linear_grad} can vanish is if $2\lambda(\norm{\vq}^2-1)=-\lambda_j$, for some $j\in [d]$. That is, $\norm{\vq_j}^2=1-\frac{\lambda_j}{2\lambda}$, and if the r.h.s.~is strictly greater than zero then $\vq_j=\sqrt{1-\frac{\lambda_j}{2\lambda}}\ve_j$ is a critical point, where $\ve_j$ is the $j$-th standard basis vector in $\Real^d$. Note that also 
$-\vq_j$, $j\in[d]$ are critical points. So in total we have found at-least $3$ and up-to $2d+1$ critical points: $\zero,\pm\vq_j$, $j\in[d]$. 

Plugging these critical solutions into the Hessian formula, \eqref{e:loss_linear_hess} we get 
$$\nabla^2_\vq\ell(\pm \vq_j)=2\diag(\lambda_1-\lambda_j,{\scriptscriptstyle\ldots},\lambda_d-\lambda_j)+8(\lambda-\frac{\lambda_j}{2}) \ve_j\ve_j^{T}.$$
From this equation we see that if $\lambda>\frac{\lambda_1}{2}$ then $\pm\vq_1$ are local minima; and for all $\lambda>0$, $\pm \vq_j$, $j\geq 2$ are saddle points or local maxima (in particular, $\vq_d$ for small $\lambda$); \ie, the Hessian $\nabla^2_\vq\ell(\vq_j)$, $j\geq 2$, has at-least one strictly negative eigenvalue. Since $\ell(\vq)\too\infty$ as $\norm{\vq}\too \infty$ we see that $\pm\vq_1$ are also global minima. 
\end{proof}

\paragraph{Convergence to global minima.} Given the classification of the critical points in Theorem \ref{thm:critical} we can prove that  
\begin{theorem}\label{thm:converge}
The gradient descent in \eqref{e:gradient_descent}, with sufficiently small step-size $\alpha>0$ and a random initialization $\vw^0$ will avoid the bad critical points of the loss function $\ell$, with probability one.
\end{theorem}
Since $\ell(\vw)\too\infty$ as $\norm{\vw}\too \infty$ and $\ell(\vw)\geq 0$ everywhere, an immediate consequence of this theorem is that \eqref{e:gradient_descent} converges (up-to the constant step-size) with probability one to one of the two global minima that represent the signed distance function to $\gH$. 

To prove Theorem \ref{thm:converge} note that Theorem \ref{thm:critical} shows that the Hessian (\eqref{e:loss_linear_hess}) evaluated at all bad critical points (\ie, excluding the two that correspond to the signed distance function) have at-least one strictly negative eigenvalue. Such saddle points are called \emph{strict saddle points}.  Theorem 4.1 in \cite{lee2016gradient} implies that gradient descent will avoid all these strict saddle points. Since the loss is non-negative and blows-up at infinity the proof is concluded. \qed

\section{Implementation and evaluation details}
\label{s:implementation}
\paragraph{Architecture.}
For representing shapes we used level sets of MLP $f\left(\vx;\theta\right)$; $f:\Real^3\times \Real^m \too \Real$, with 8 layers, each contains 512 hidden units, and a single skip connection from the input to the middle layer as in \cite{park2019deepsdf}. The weights $\theta\in\Real^m$ are initialized using the geometric initialization from \cite{atzmon2020sal}. We set our loss parameters (see \eqref{e:loss}) to  $\lambda=0.1,\tau=1$.

\paragraph{Distribution $\gD$.} We defined the distribution $\gD$ for the expectation in \eqref{e:loss} as the average of a uniform distribution and a sum of Gaussians centered at $\gX$ with standard deviation equal to the distance to the $k$-th nearest neighbor (we used $k=50$). This choice of $\gD$ was used throughout all experiments. 

\paragraph{Level set extraction.}
We extract the zero (or any other) level set of a trained MLP $f(\vx;\theta)$ using the \emph{Marching Cubes} meshing algorithm \cite{lorensen1987marching} on a uniform sampled grids of size $\ell^3$, where $\ell\in\set{256,512}$.

\paragraph{Evaluation metrics.}
Our quantitative evaluation is based on the following collection of standard metrics of two point sets $\gX_1,\gX_2\subset\Real^3$: the Chamfer and Hausdorff distances, \begin{align} \label{e:CD}
\dist_{\text{C}}\left(\gX_{1},\gX_{2} \right) & = \frac{1}{2}\left(\dist_{\text{C}}^{\rightarrow}\left(\gX_{1},\gX_{2} \right) + \dist_{\text{C}}^{\rightarrow}\left(\gX_{2},\gX_{1} \right) \right) \\ \label{e:HD}
\dist_{\text{H}}\left(\gX_{1},\gX_{2} \right) & =  \max \left\{ \dist_{\text{H}}^{\rightarrow}\left(\gX_{1},\gX_{2} \right),\dist_{\text{H}}^{\rightarrow}\left(\gX_{2},\gX_{1} \right) \right\}
\end{align}
where 
\begin{align}\label{e:one_sided_CD}
\dist_{\text{C}}^{\rightarrow}\left(\gX_{1},\gX_{2} \right) & = \frac{1}{\abs{\gX_1}}\sum_{\vx_{1}\in\gX_{1}}\min_{\vx_2\in \gX_{2}}\norm{\vx_1-\vx_2} , \\ \label{e:one_sided_HD}
\dist_{\text{H}}^{\rightarrow}\left(\gX_{1},\gX_{2} \right) & =  \max_{\vx_{1}\in\gX_{1}}\min_{\vx_2\in \gX_{2}}\norm{\vx_1-\vx_2}
\end{align}
are the one-sided Chamfer and Hausdorff distances (resp.).

\section{Model evaluation}

\paragraph{Signed distance function approximation.}
We start our evaluation by testing the ability of our trained model $f$ to reproduce a signed distance function (SDF) to known manifold surfaces. We tested: a plane, a sphere, and the Bimba model. In this experiment we used no normals and took the sample point cloud $\gX$ to be of infinite size (\ie, draw fresh point samples every iteration). 

\begin{wraptable}[6]{r}{0.45\columnwidth}
    \small
    \vspace{-10pt}\hspace{-5pt}
    \begin{tabular}{lr} 
                    & Relative Error                 \\
        \midrule
        Plane    & ${0.003 \pm 0.04}$       \\
        Sphere    & ${0.004 \pm 0.08}$      \\
        Bimba   & ${0.008 \pm 0.11}$        \\
        \end{tabular} \vspace{-8pt}
        \caption{SDF approximation.}
       \label{tab:sdf_approx}
\end{wraptable}
For each surface we separately train an MLP $f(\vx;\theta)$ with sampling distribution $\gD$. Table \ref{tab:sdf_approx} logs the results, where we report mean $\pm$ std of the relative error measured at $100$k random points. The relative error is defined by  $\frac{\abs{f\left(\vx;\theta \right) - s\left(\vx\right)}}{{\abs{s\left(\vx\right)}}}$, where $s:\Real^{3} \too \Real$ is the ground truth signed distance function. 
Figure \ref{fig:3models} provides a visual validation of the quality of our predictions, where equispaced positive (red) and negative (blue) level sets of the trained $f$ are shown; the zero level sets are in white. \vspace{-10pt}
\begin{figure}[t]
 \includegraphics[width=1.0\columnwidth]{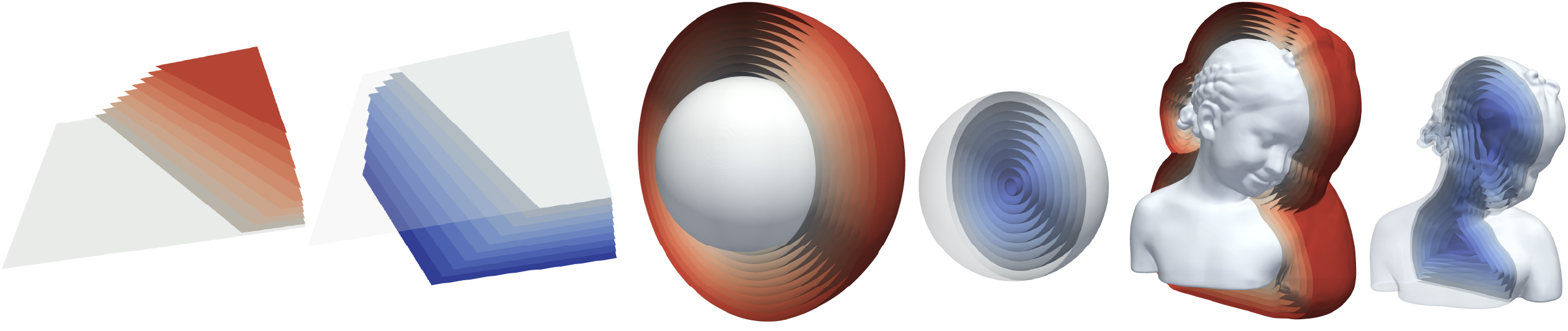}\vspace{0pt}
    \caption{Level sets of MLPs trained with our method.}
    \label{fig:3models}
    \vspace{-10pt}
\end{figure}

\paragraph{Fidelity and level of details.}
As mentioned above, previous works have suggested learning shapes as level sets of implicit neural networks (see \eqref{e:gM}) via regression or classification \cite{park2019deepsdf,mescheder2019occupancy,chen2019learning}. 

To test the faithfulness or fidelity of our learning method in comparison to regression we considered two raw scans (\ie, triangle soups) of a human, $\gS$, from the D-Faust \cite{dfaust:CVPR:2017} dataset. For each model we took a point sample $\gX\subset \gS$ of $250$k points, and a corresponding normal sample $\gN$ (from the triangles). We used this data $\gX,\gN$ to train an MLP with our method. 

For regression, we trained an MLP with the same architecture using an approximated SDF data pre-computed using a standard local SDF approximation. Namely, $s\left(\vx\right) = \vn_*^T\left(\vx - \vy_*\right)$, where $\vy_* = \argmin_{\vy\in\gS}\norm{\vy-\vx}_2$ and $\vn_*$ is the unit normal of the triangle containing $\vy_*$. We trained the MLP with an $L_1$ regression loss $\E_{\vx\sim \gD'}\abs{f\left(\vx;\theta\right) - s\left(\vx\right)}$, where $\gD'$ is a discrete distribution of $500$k points ($2$ points for every point in $\gX$) defined as in \cite{park2019deepsdf}.
Figure \ref{fig:lod} shows the zero level sets of the trained networks. Note that our method produced considerably more details than the regression approach. This improvement can be potentially attributed to two properties: First, our loss incorporates only points contained in the surface, while regression approximates the SDF nearby the actual surface. Second, we believe the implicit regularization property improves the fidelity of the learned level sets. 

\begin{figure}[t]
    \centering
    \begin{tabular}{c}
     \includegraphics[width=1.0\columnwidth]{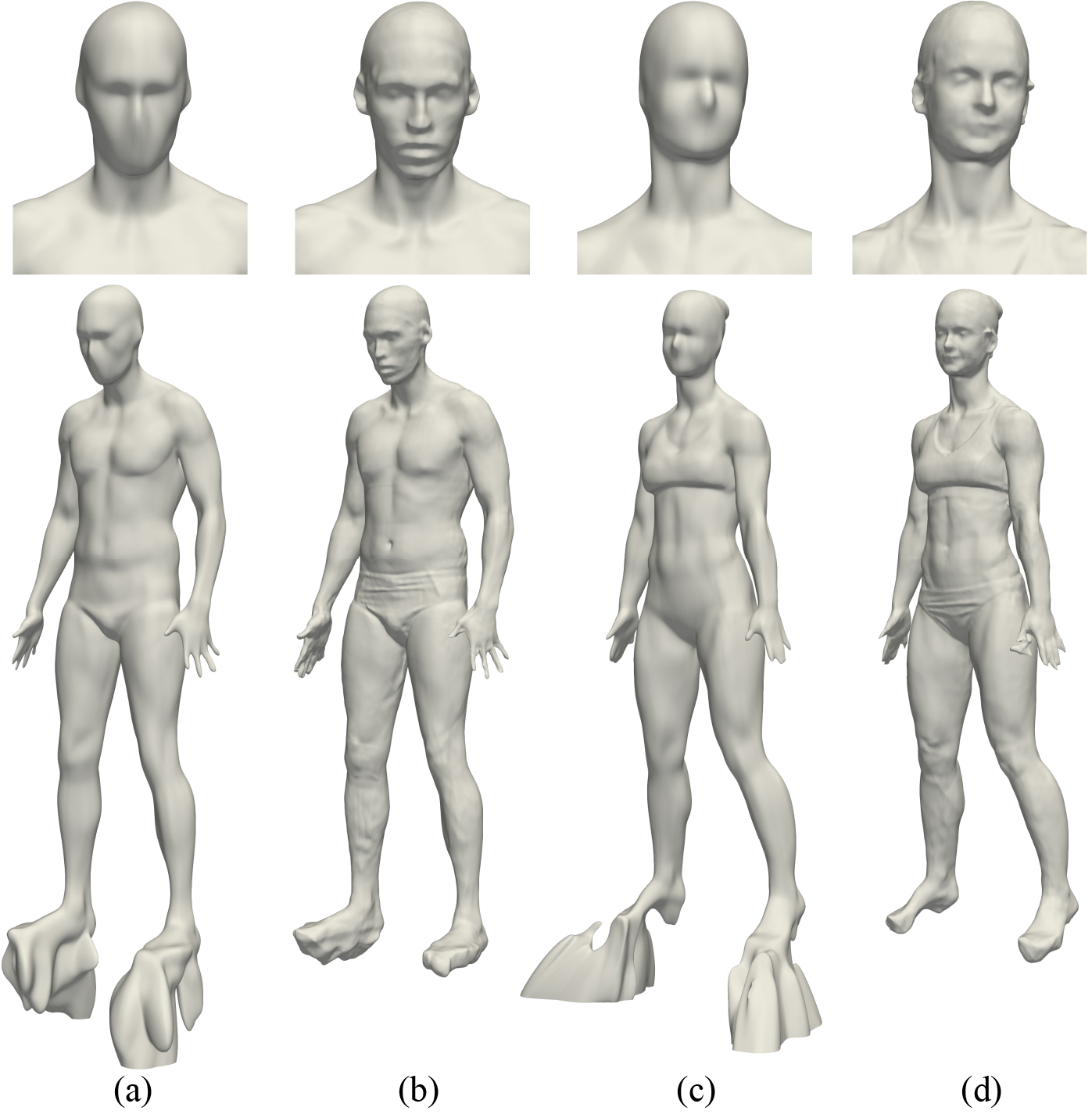}\vspace{-10pt}
\end{tabular}
    \caption{Level of details comparison. The zero level sets of an MLP trained with our method in (b) and (d); and using regression loss in (a) and (c), respectively.  }
    \label{fig:lod}
    \vspace{-5pt}
\end{figure}

\section{Experiments}

\begin{table}[t]
    \centering
    \small
    \setlength\tabcolsep{6pt} 
    \begin{tabular}{c}
        \begin{adjustbox}{max width=\textwidth}
            \aboverulesep=0ex
            \belowrulesep=0ex
            \begin{tabular}[t]{c|c|c|c | c | c|}
            \multicolumn{2}{c}{} & 
            \multicolumn{2}{|c|}{Ground Truth} & 
            \multicolumn{2}{|c|}{Scans} \\
                & Method                                  &
                $\dist_{\text{C}}$  & $\dist_{\text{H}}$ & $\dist_{\text{C}}^{\rightarrow}$ & $\dist_{\text{H}}^{\rightarrow}$ \\
                \midrule
                \multirow{ 2}{*}{Anchor} &
                DGP & 0.33 & 8.82 & \textbf{0.08} & 2.79 \\
                & Ours &  \textbf{0.22} & \textbf{4.71} & 0.12 & \textbf{1.32} \\
                 \midrule
                 \multirow{ 2}{*}{Daratech} &
                 DGP &  \textbf{0.2} & \textbf{3.14} &  \textbf{0.04 } & 1.89 \\
                & Ours & 0.25 & 4.01 & 0.08 & \textbf{1.59} \\
                \midrule
                 \multirow{ 2}{*}{Dc} &
                 DGP & 0.18 & 4.31 &  \textbf{0.04} & \textbf{2.53} \\
                & Ours & \textbf{0.17} &  \textbf{2.22} & 0.09 & 2.61 \\
                \midrule
                 \multirow{ 2}{*}{Gargoyle} &
                 DGP &  0.21 & 5.98 & \textbf{0.062} & 3.41 \\
                & Ours &  \textbf{0.16} & \textbf{3.52} & 0.064 & \textbf{0.81} \\
                \midrule
                 \multirow{ 2}{*}{Lord Quas} &
                DGP & 0.14 & 3.67 & \textbf{0.04} & 2.03 \\
                & Ours &  \textbf{0.12} & \textbf{1.17} & 0.07 & \textbf{0.98} \\
        \end{tabular} 
        \end{adjustbox}
      
    \end{tabular}
    \caption{Evaluation on the surface reconstruction benchmark versus DGP \cite{williams2019deep}. } \vspace{-5pt}
    \label{tab:surface_reconstruction}
\end{table}

\begin{figure}[t]
      \includegraphics[width=1.0\columnwidth]{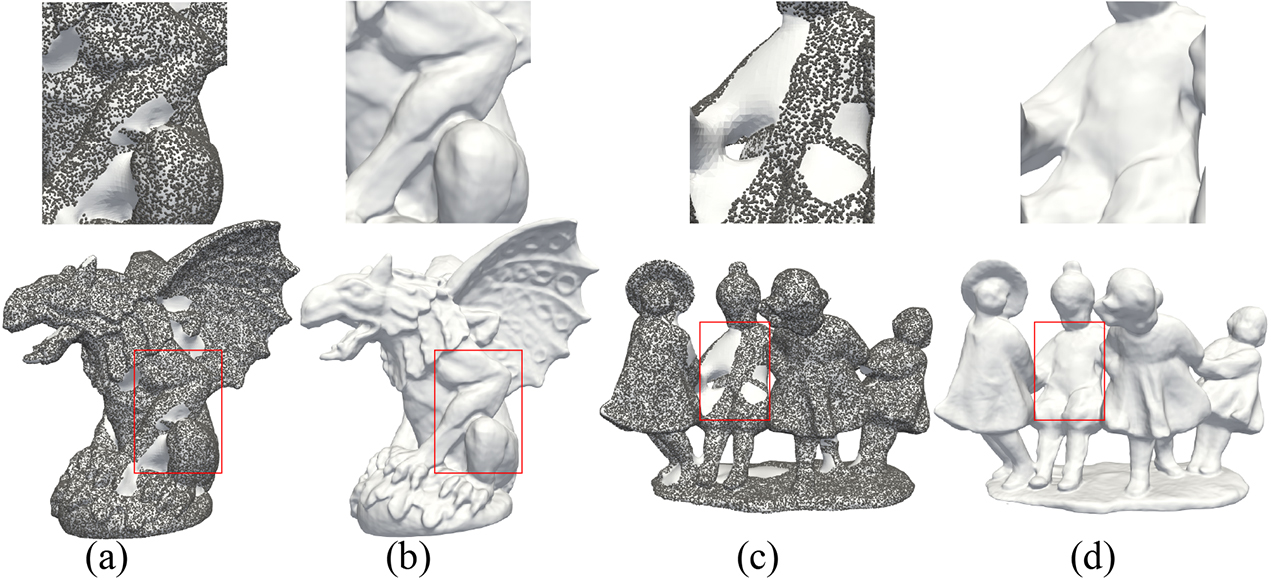}\vspace{-10pt}
    \caption{Reconstructions with our method in (b) and (d) versus \cite{williams2019deep} (DGP) in (a) and (c). Note the charts of DGP do not cover the entire surface area.}
    \label{fig:comparison_to_DGP}
    \vspace{0pt}
\end{figure}

\begin{figure}[h]
      \includegraphics[width=1.0\columnwidth]{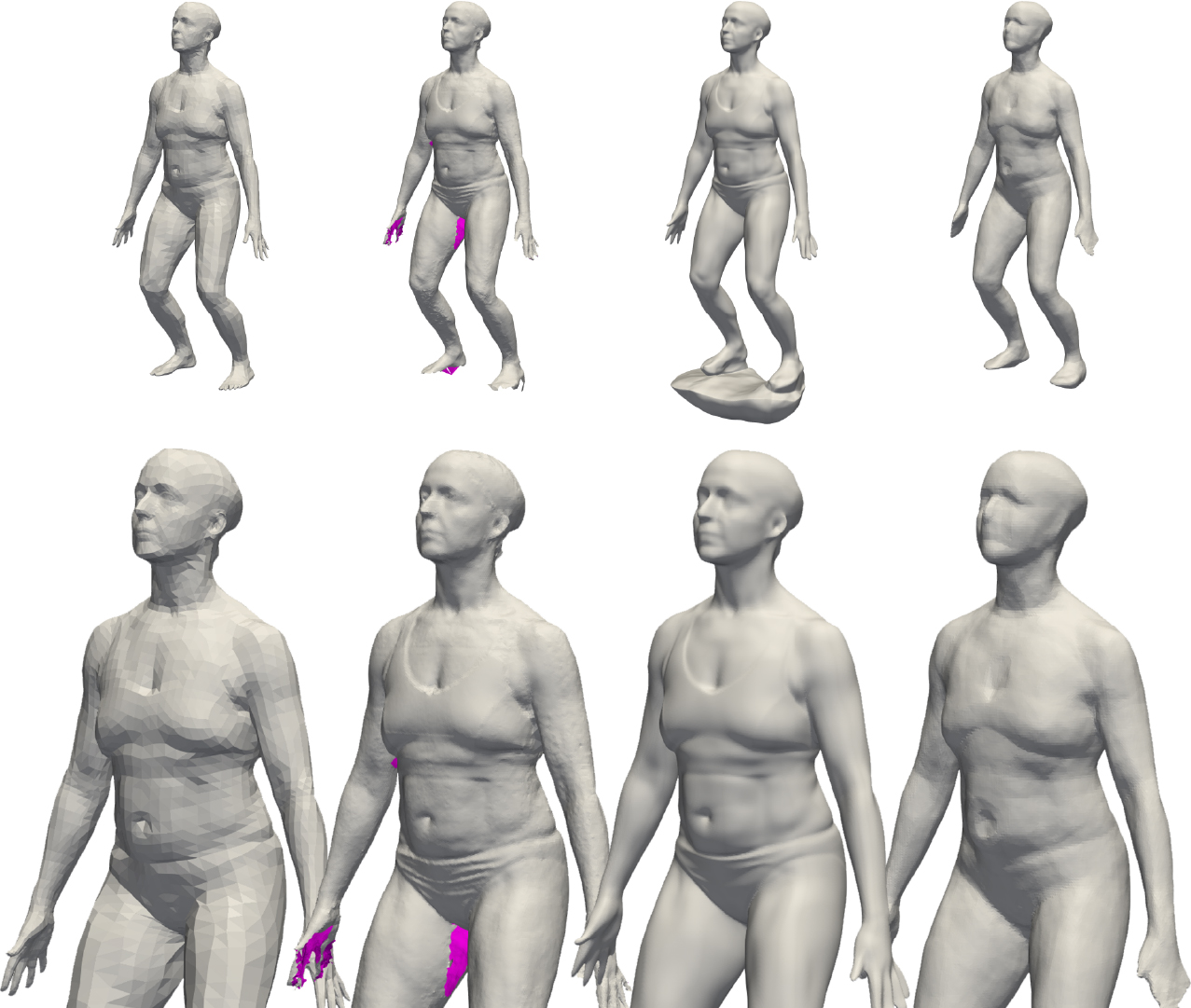}\vspace{-10pt}
    \caption{A train result on D-Faust. Left to right: registrations, scans, our results, SAL.}\vspace{-5pt}
    \label{fig:dfaust_short_train}
\end{figure}

\subsection{Surface reconstruction}\label{s:surface_recon}
We tested our method on the task of surface reconstruction. That is, given a \emph{single} input point cloud $\gX\subset \Real^3$ with or without a set of corresponding normal vectors $\gN\subset \Real^3$, the goal is to approximate the surface that $\gX$ was sampled from. The sample $\gX$ is usually acquired using a 3D scanner, potentially introducing variety of defects and artifacts. 
We evaluated our method on the surface reconstruction benchmark \cite{berger2013benchmark}, using data (input point clouds $\gX$, normal data $\gN$, and ground truth meshes for evaluation) from \cite{williams2019deep}. The benchmark consists of five shapes with challenging properties such as non trivial topology or details of various feature sizes including sharp features. We compared our performance to  the method from \cite{williams2019deep} (DGP), which is a recent deep learning chart-based surface reconstruction technique; \cite{williams2019deep} also provide plethora of comparisons to other surface reconstruction techniques and establish itself as state of the art method. Table \ref{tab:surface_reconstruction} logs the performance of both methods using the following metrics: we measure distance of reconstructions to ground truth meshes using the (two-sided) Chamfer distance $d_C$ and the (two-sided) Hausdorff distance $d_H$; and distance from input point clouds to reconstructions using the (one-sided) Chamfer distance $\dist_{\text{C}}^{\rightarrow}$ and the (one-sided) Hausdorff distance $\dist_{\text{H}}^{\rightarrow}$. Our method improves upon DGP in 4 out of 5 of the models in the dataset when tested on the ground truth meshes. DGP provides a better fit in average to the input data $\gX$ (our method performs better in Hausdorff); this might be explained by the tendency of DGP to leave some uncovered areas on the surface, see \eg, Figure \ref{fig:comparison_to_DGP} highlighting uncovered parts by DGP.

\begin{figure}[h]
      \includegraphics[width=1.0\columnwidth]{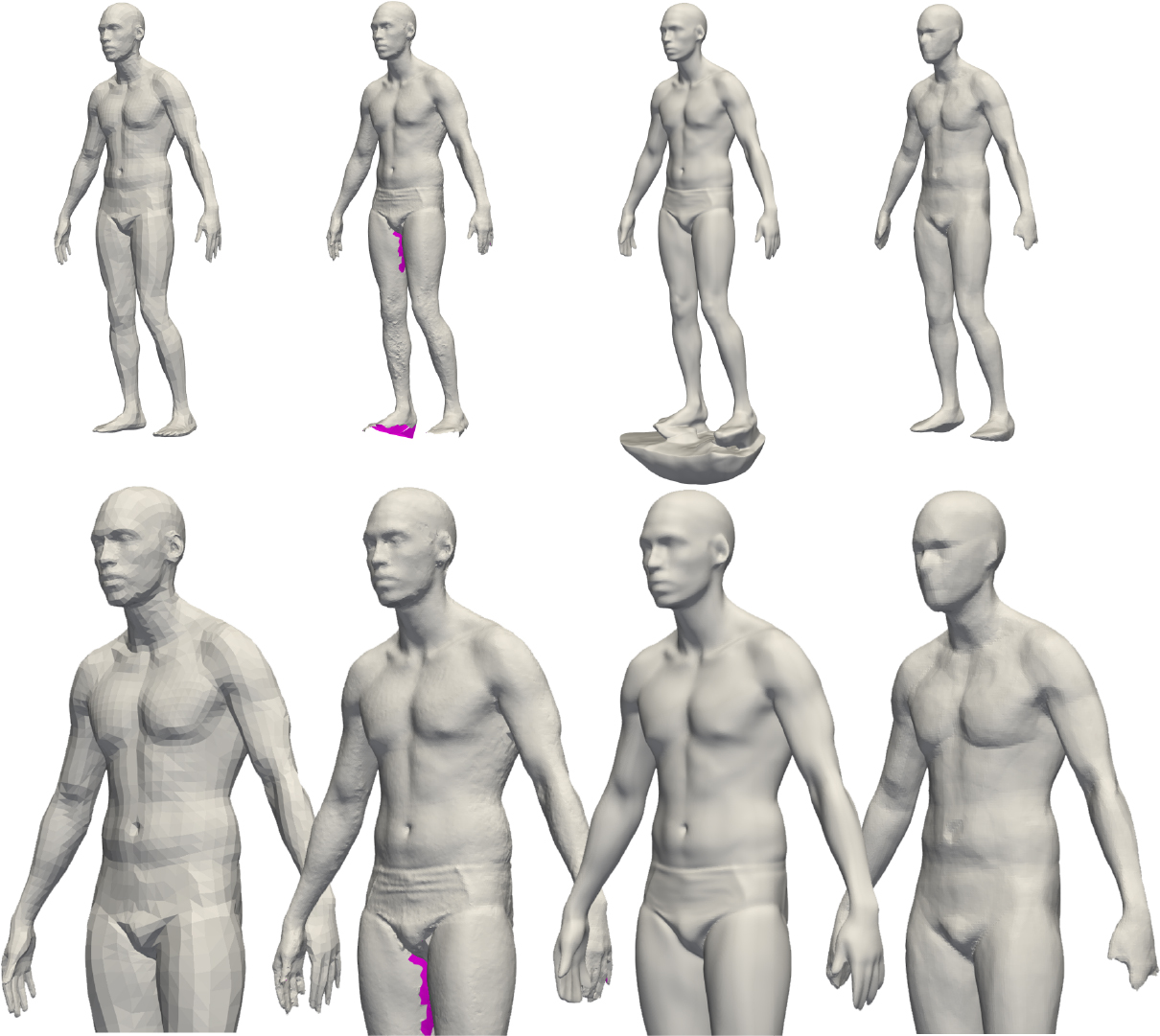}\vspace{-10pt}
    \caption{A test result on D-Faust. Left to right: registrations, scans, our results, SAL.}
    \label{fig:dfaust_short_test}
\end{figure}

\begin{figure}[h]
      \includegraphics[width=1.0\columnwidth]{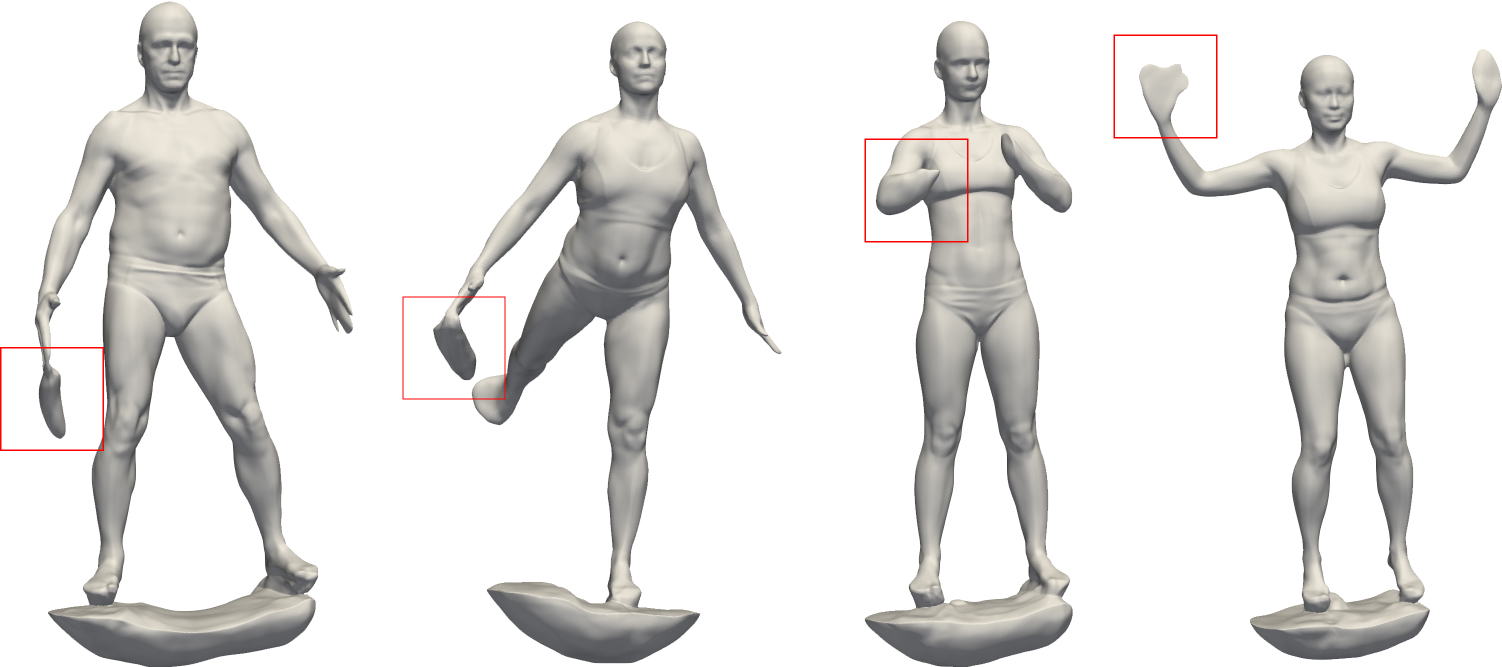}\vspace{-13pt}
    \caption{Failures of our method on D-Faust. }
    \label{fig:dfaust_failures}
\end{figure}

\subsection{Learning shape space}\label{ss:shape_spaces}
In this experiment we tested our method on the task of learning shape space from raw scans. To this end, we use the D-Faust \cite{dfaust:CVPR:2017} dataset, containing high resolution raw scans (triangle soups) of 10 humans in multiple poses. For training, we sampled each raw scan, $\gS_j$, $j\in J$, uniformly to extract point and normal data,  $\left\{(\gX_j,\gN_j\right)\}_{j\in J}$.
We tested our method in two different settings: (i) random 75\%-25\% train-test split; (ii) generalization to unseen humans - $8$ out of $10$ humans are used for training and the remaining $2$ for testing. In both cases we used the same splits as in \cite{atzmon2020sal}. The second column from the left in Figures \ref{fig:dfaust_short_train}, \ref{fig:dfaust_short_test} show examples of input \emph{scans} from the D-Faust datset. The left column in these figures shows the ground truth \emph{registrations}, $\gR_j$, $j\in J$, achieved using extra data (\eg, color and texture) and human body model fitting \cite{dfaust:CVPR:2017}. Note that we \emph{do not} use the registrations data in our training, we only use the raw scans. 

\paragraph{Multi-shape architecture.}
In order to extend the network architecture described in section \ref{s:implementation} for learning multiple shapes, we use the auto-decoder setup as in \cite{park2019deepsdf}. That is, an MLP $f(\vx;\theta,\vz_j)$, where $\vz_j\in\Real^{256}$ is a latent vector corresponding to each training example $j\in J$.  The latent vectors $\vz_j\in\Real^{256}$ are initialized to $\vzero\in\Real^{256}$. We optimize a loss of the form
$\frac{1}{|B|}\sum_{j\in B}\ell(\theta,\vz_j) + \alpha\norm{\vz_j}$, where $B\subset J$ is a batch, $\alpha=0.01$, $\ell$ defined in \eqref{e:loss}; $\tau,\lambda$ as above.


\begin{figure}[t]
\begin{tabular}{@{}c@{}c@{}c@{}}
    \includegraphics[width=0.33\columnwidth]{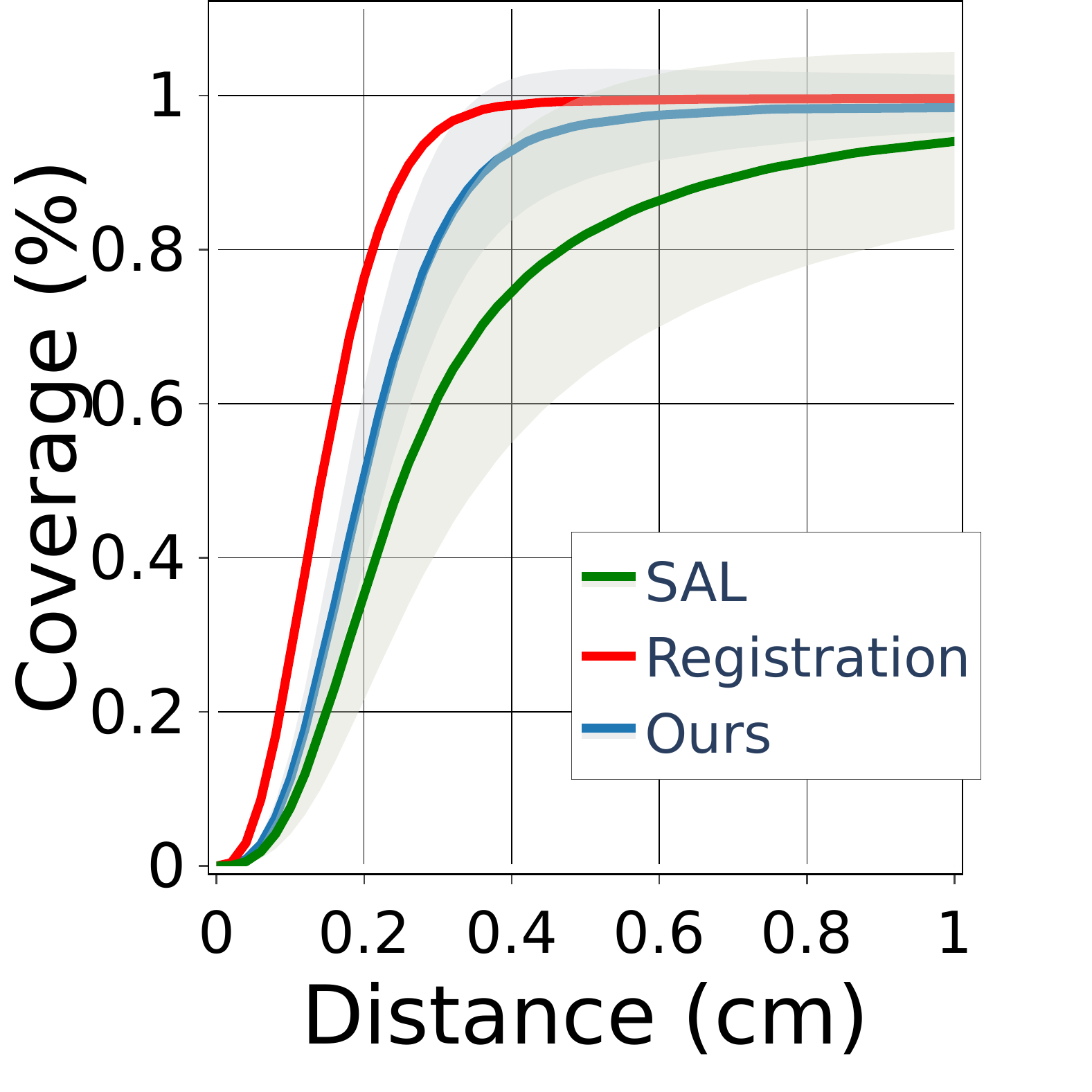} &
     \includegraphics[width=0.33\columnwidth]{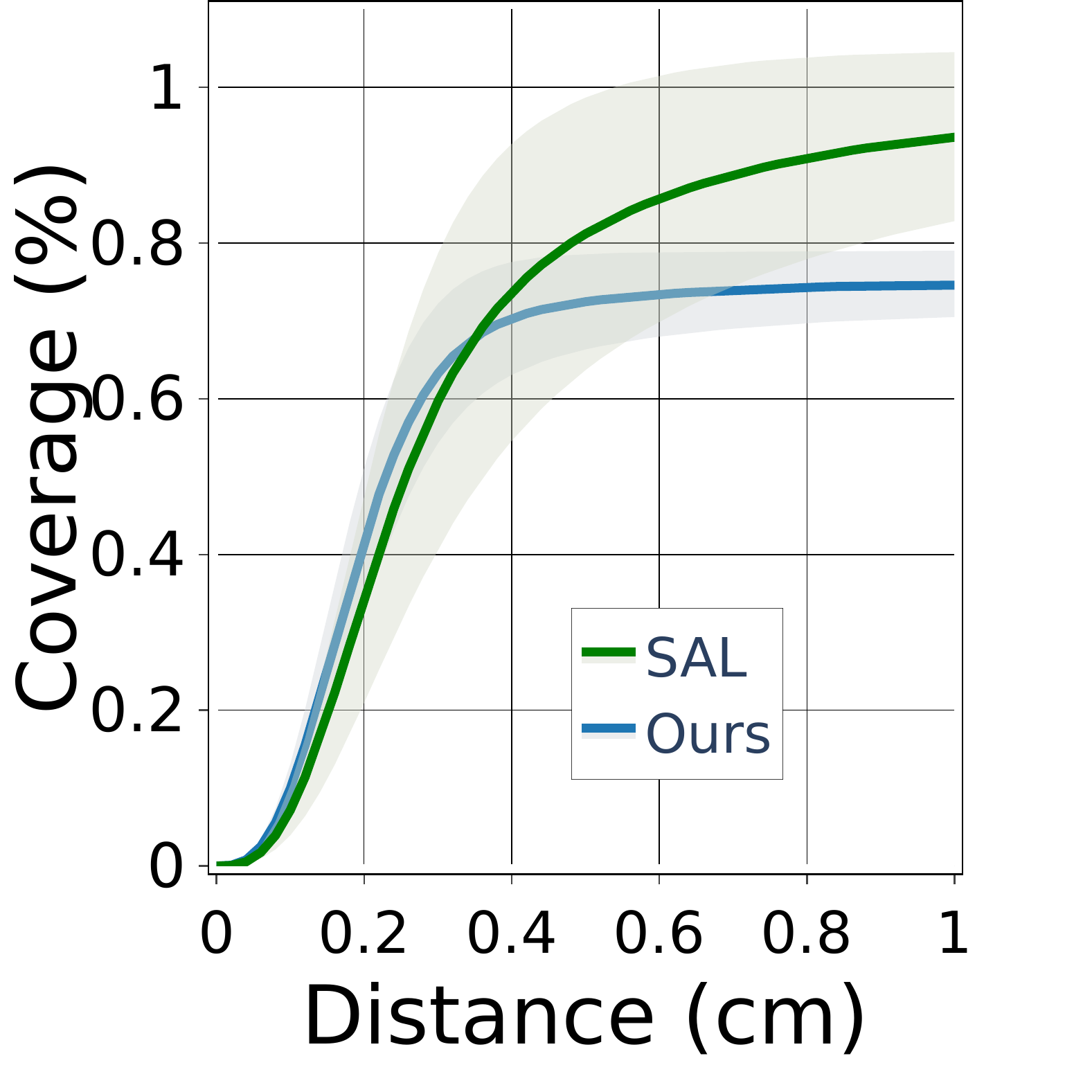} &
     \includegraphics[width=0.33\columnwidth]{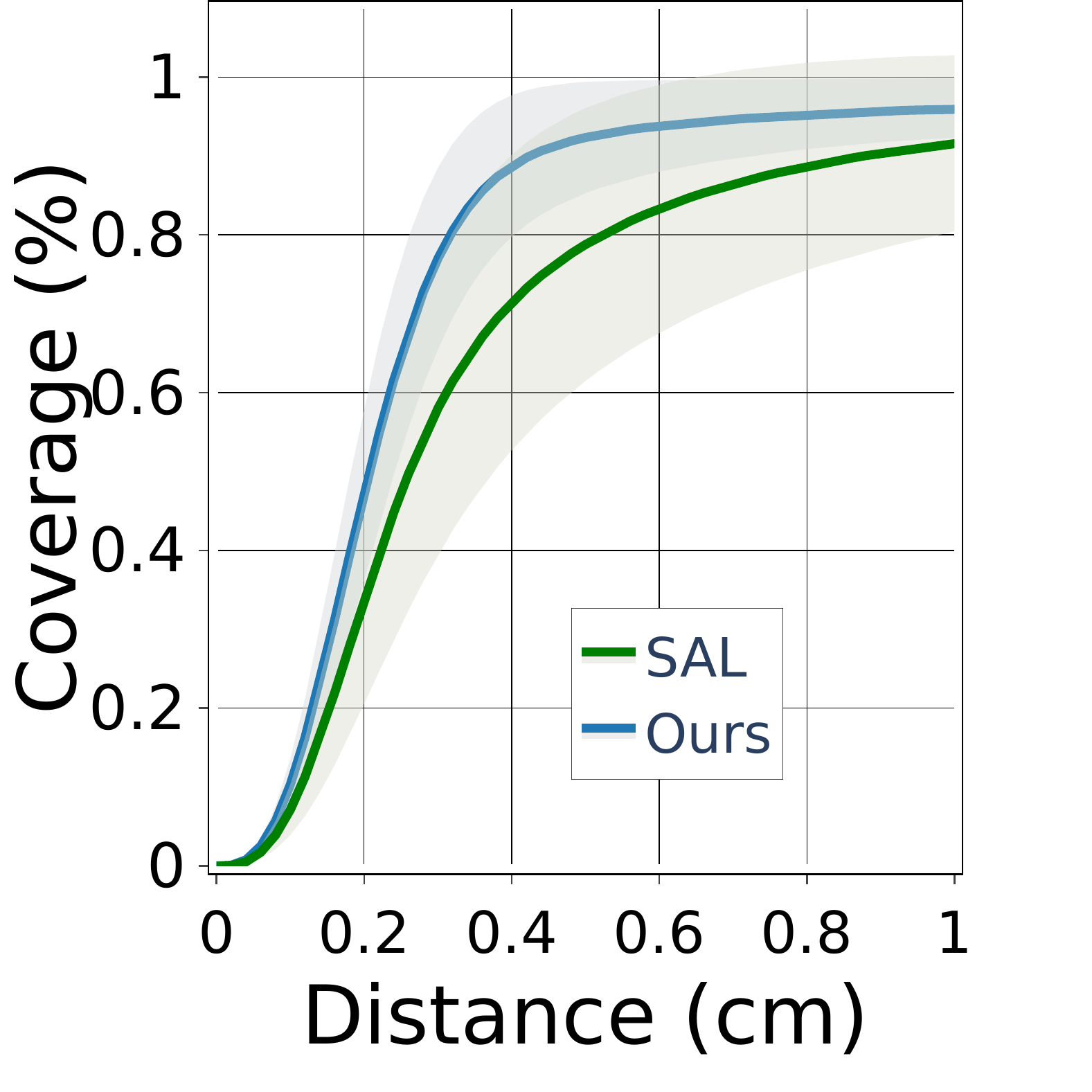} \\
    \includegraphics[width=0.33\columnwidth]{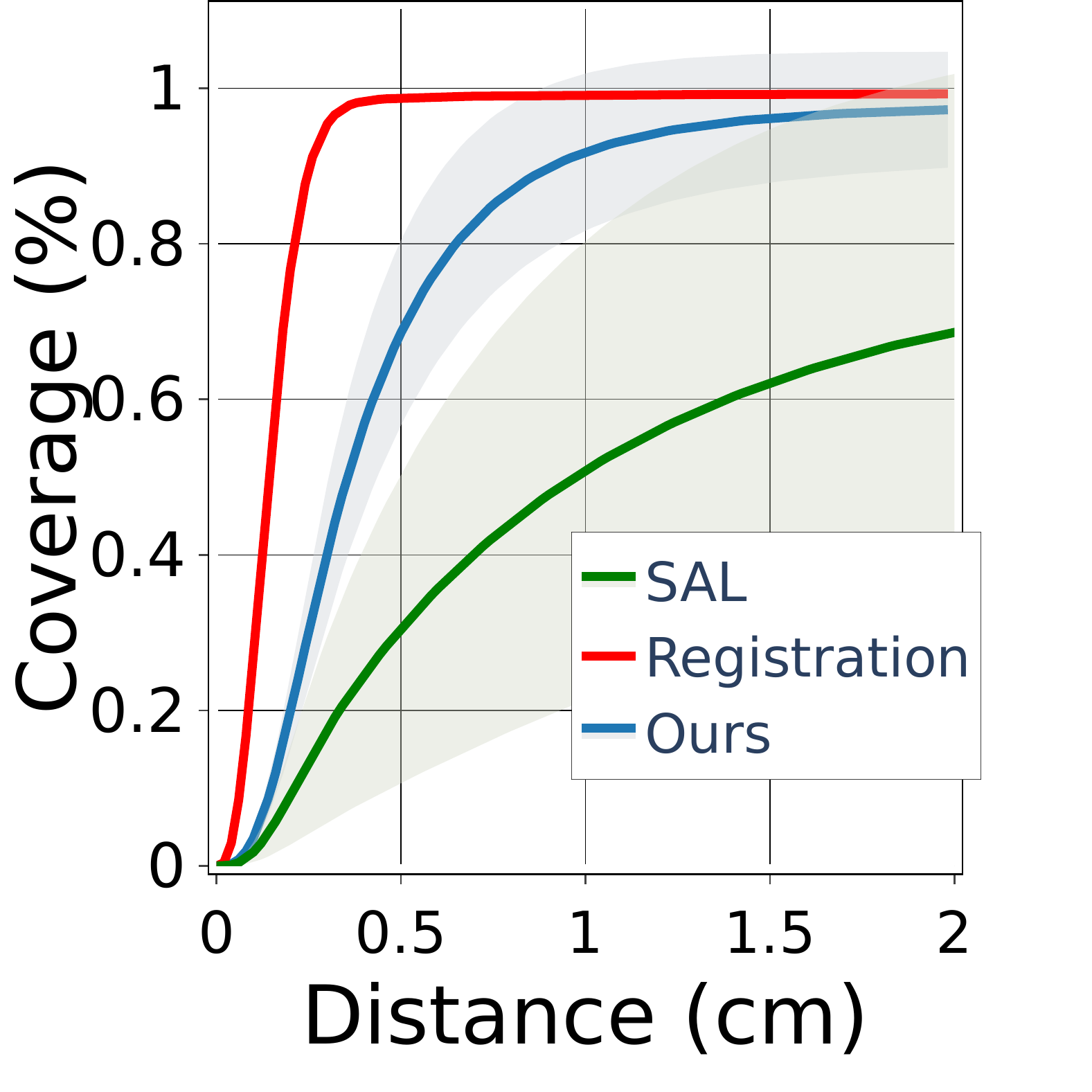} &
     \includegraphics[width=0.33\columnwidth]{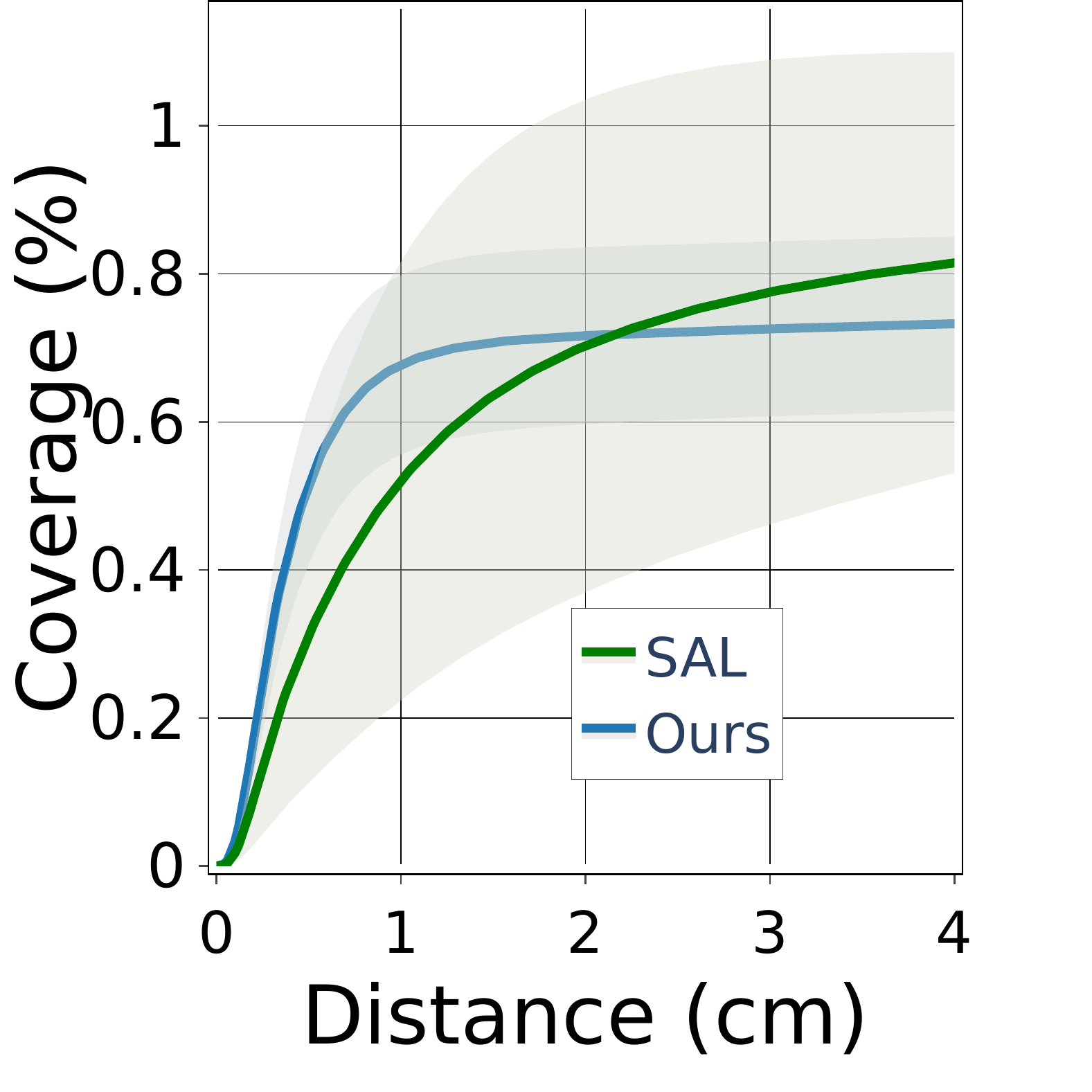} &
     \includegraphics[width=0.33\columnwidth]{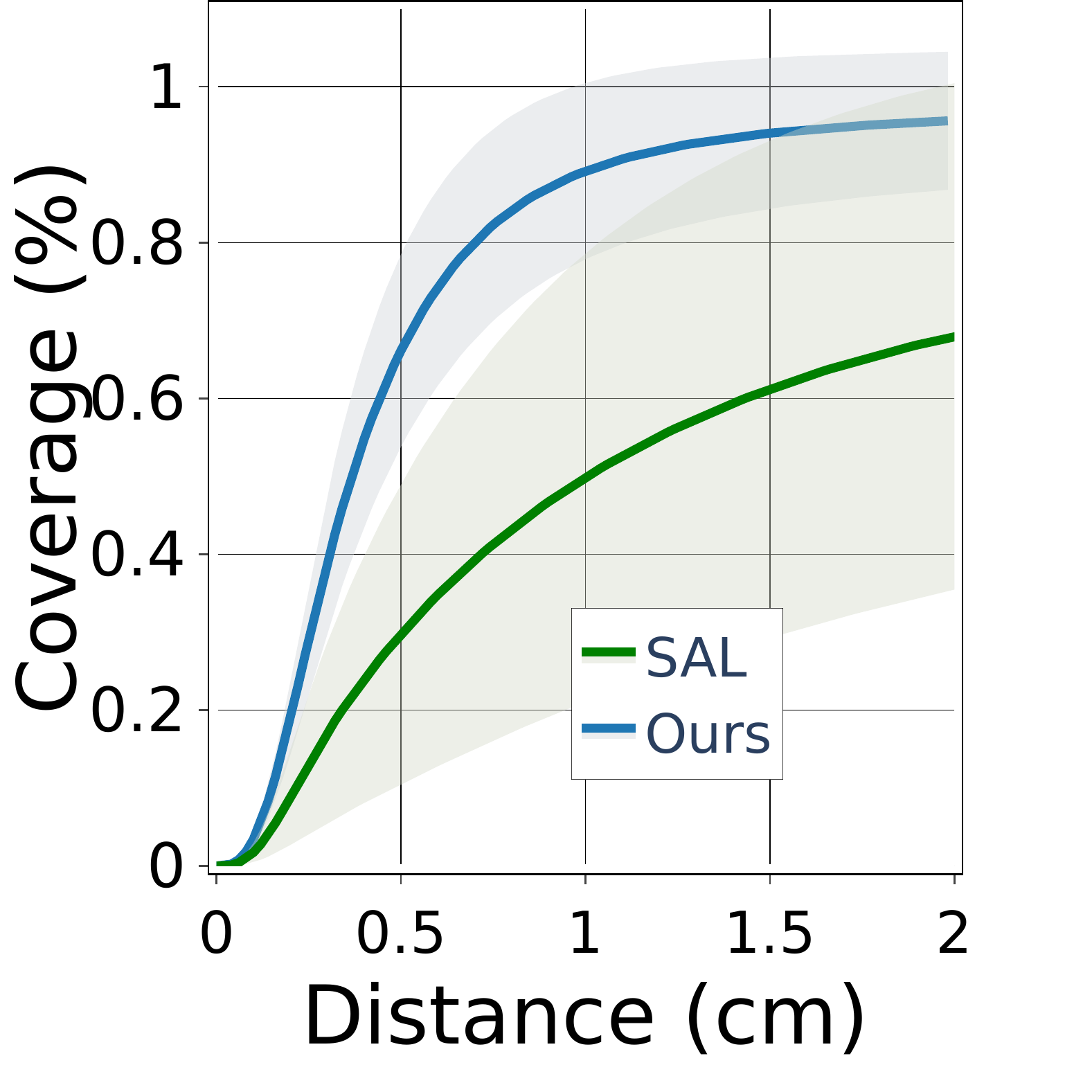} \\
     (a) & (b) & (c) 
\end{tabular}
\caption{Error versus coverage for D-Faust test shapes with random split (first row) and unseen humans split (second row):
\\(a) one-sided Chamfer distance from scan-to-reconstruction;
\\(b) one-sided Chamfer distance reconstruction-to-registration;\\
(c) one-sided Chamfer distance registration-to-reconstruction.\\
}
\label{fig:dfaust_res}
\vspace{-12pt}
\end{figure}

\begin{figure}[b!]
    \centering
     \includegraphics[width=\columnwidth]{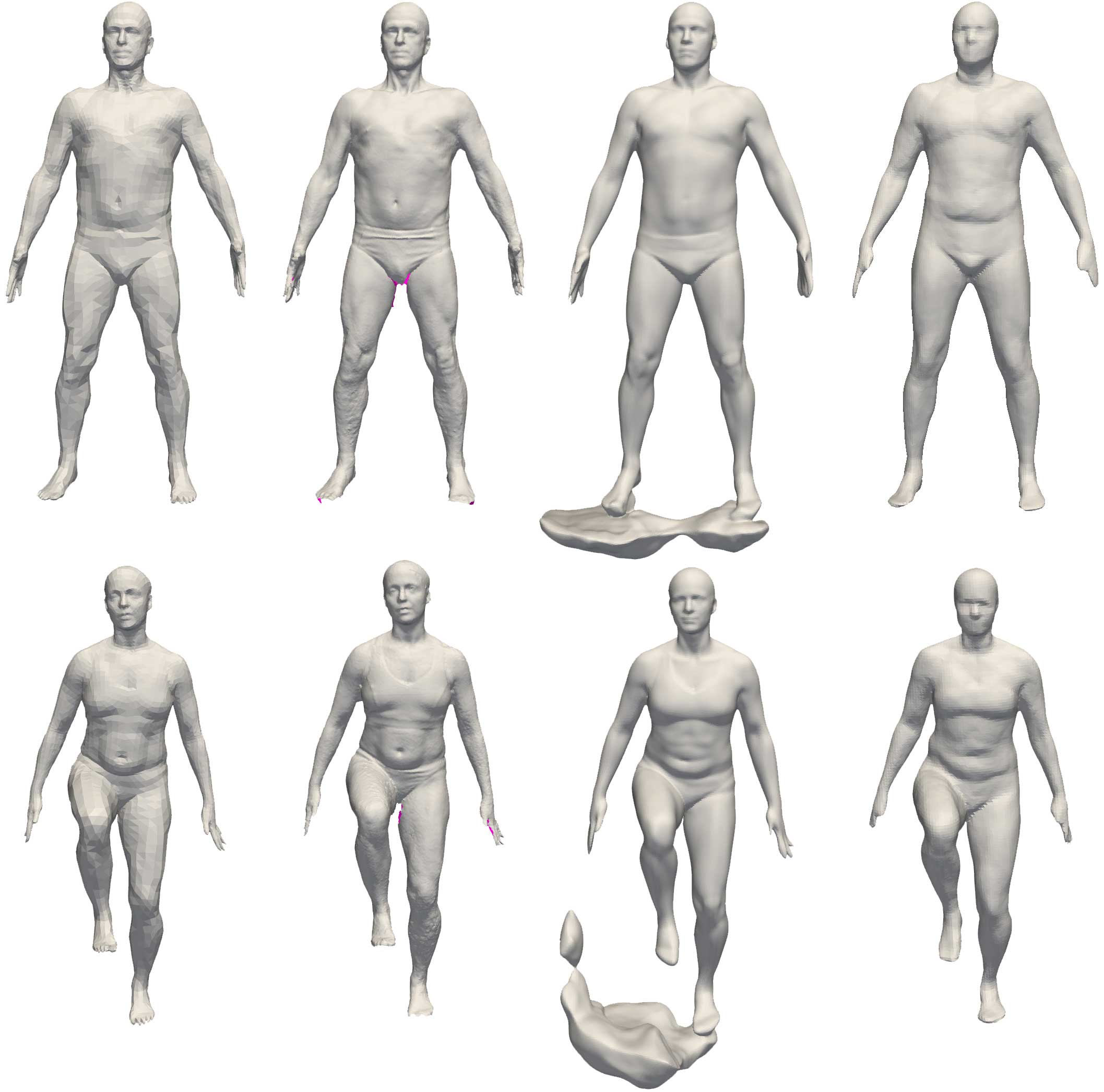}
    \vspace{-12pt}
    \caption{
    A test result on D-Faust with unseen humans split. Left to right: registrations, scans, our results, SAL.}
    \label{fig:humans}
\end{figure}

\begin{figure*}[t!]
    \centering
     \includegraphics[width=\textwidth]{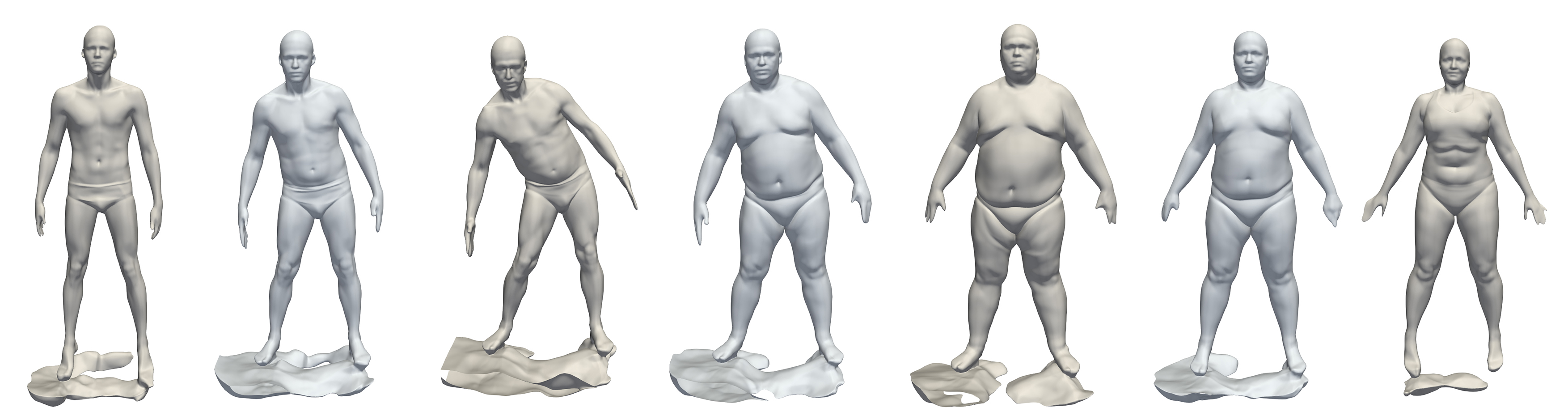}
    \caption{Averaged shapes: Zero level sets (in blue) using averages of latent vectors of train examples (in gray).}
    \label{fig:interpolation}
\end{figure*}

\paragraph{Evaluation.}
For evaluation, we used our trained model for predicting shapes on the held out test set. As our architecture is an auto-decoder, the prediction for each test shape is obtained by performing $800$ iterations of latent vector $\vz$ optimization of the loss $\ell$. We compared our results versus those obtained using SAL \cite{atzmon2020sal}, considered as state-of-the-art learning method on this dataset. 

Figures \ref{fig:dfaust_short_train} and \ref{fig:dfaust_short_test} show examples from the train and test random split, respectively. Results for the unseen humans experiment are shown in \ref{fig:humans}. More results can be found in the appendix. Magenta triangles are back-faces hence indicating holes in the scan. Note that our method produces high level reconstructions with more details than SAL. In the unseen human tests the method provides plausible approximation despite training on only 8 human shapes. We note that our method produces a sort of a common "base" to the models, probably due to parts of the floor in some of the train data. 

Figure \ref{fig:dfaust_res} quantifies coverage as a function of error: Given a point cloud $\gX$, we measure, for each distance value $\epsilon$ (X-axis),  
the fraction of points $\vx\in\gX$ that satisfy $\dist(\vx,\gY)<\epsilon$. In (a) $\gX \subset \gS_j$, and $\gY$ are the reconstructions of the registration, SAL and our method; note that our reconstructions are close in performance to the ground truth registrations. In (b), $\gX$ is a sampling of the reconstruction of SAL and our method, $\gY=\gR_j$. In (c), $\gX$ is a sampling of $\gR_j$ and $\gY$ is the reconstructions of SAL and our method. The lines represent mean over $j\in J$ and shades represent standard deviations. Note that we improve SAL except for larger errors in (b), a fact which can be attributed to the "base" reconstructed by our method. Some failure examples are shown in \ref{fig:dfaust_failures}, mainly caused by noisy normal data. In addition, note that for the unseen humans split, we get relatively higher reconstruction error rate than the random split. We attribute this to the fact that there are only $8$ humans in the dataset.  

\paragraph{Shape space exploration.}
For qualitative evaluation of the learned shape space, we provide reconstructions obtained by interpolating latent vectors. Figure \ref{fig:interpolation} shows humans shapes (in blue) corresponding to average interpolation of latent vectors $\vz_j$ of training examples (in gray). Notice that the averaged shapes nicely combine and blend body shape and pose.

\section{Conclusions}
We introduced a method for learning high fidelity implicit neural representations of shapes directly from raw data. The method builds upon a simple loss function. Although this loss function possesses an infinite number of signed distance functions as minima, optimizing it with gradient descent tends to choose a favorable one. We analyze the linear case proving convergence to the approximate signed distance function, avoiding bad critical points. Analyzing non-linear models is a very interesting future work, \eg, explaining the reproduction of lines and circles in \ref{fig:2d}. \\
The main limitation of the method seems to be sensitivity to noisy normals as discussed in section \ref{ss:shape_spaces}. We believe the loss can be further designed to be more robust to this kind of noise.\\
Practically, the method produces neural level sets with significant more details than previous work. An interesting research venue would be to incorporate this loss in other deep 3D geometry systems (\eg, differentiable rendering and generative models).


\section*{Acknowledgments}
The research was supported by the European Research Council (ERC Consolidator Grant, "LiftMatch" 771136), the Israel Science Foundation (Grant No. 1830/17) and by a research grant from the Carolito Stiftung (WAIC).
\bibliography{IGR-ArXiv}
\bibliographystyle{icml2020}

\appendix

\section{Additional Implementation Details}

\subsection{Network Architecture.}
We used Auto-Decoder network architecture proposed in \cite{park2019deepsdf}, as described in sections \ref{s:implementation} and \ref{ss:shape_spaces}:
\begin{figure}[h]
    \centering
    \begin{tabular}{c}
     \includegraphics[width=0.95\columnwidth]{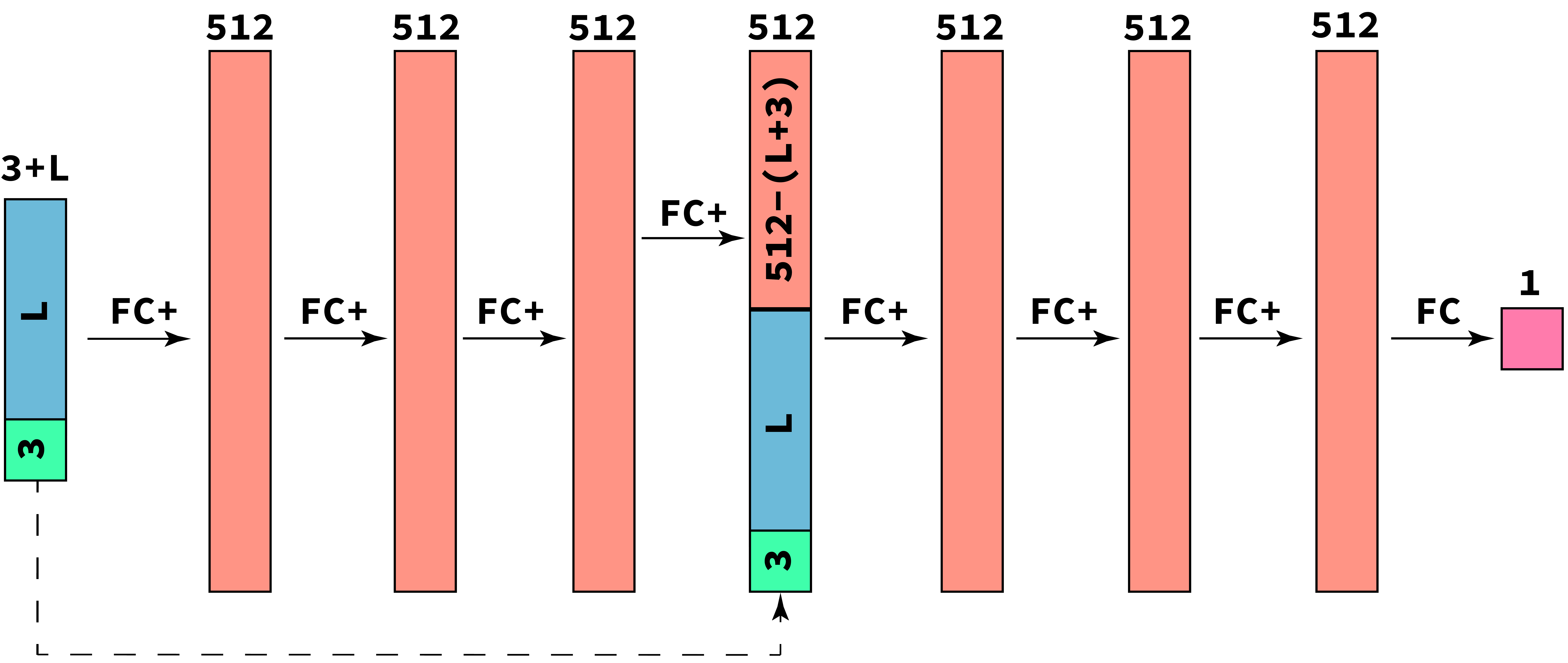}
\end{tabular}
    \label{fig:net}
    \vspace{-10pt}
\end{figure}

where $\mathbf{FC} $ is a fully connected linear layer and $\mathbf{FC+} $ is $\mathbf{FC}$ followed by $ \mathbf{softplus} $ activation; a smooth approximation of $\mathbf{ReLU}$: $ x\mapsto\frac{1}{\beta}\ln(1+e^{\beta x}) $. We used $\beta=100$. The dashed line connecting the input to the $4$th layer indicates a skip connection. $\mathbf{L}$ is the latent vector's size. For shape reconstruction application we take $\mathbf{L}=0$; for the shape space experiment we used $\mathbf{L}=256$.

\subsection{Training Details.}
\paragraph{Shape Reconstruction.}
 Training was done on a single Nvidia V-100 GPU with \textsc{pytorch} deep learning framework \cite{paszke2017automatic}. We used \textsc{Adam} optimizer \cite{kingma2014adam} for $100$k iterations with constant learning rate of $ 0.0001 $. In each iteration we sampled uniformly at random $128^2$ points from of the input point cloud.

\paragraph{Shape Space Learning.}
Training was done on 4 Nvidia V-100 GPUs, with \textsc{pytorch} deep learning framework \cite{paszke2017automatic}. We used \textsc{Adam} optimizer \cite{kingma2014adam} for $1$k epochs with initial learning rate of $ 0.0005 $ scheduled to decrease by a factor of $2$ every $500$ epochs. We divided the training set into mini-batches: a batch contains $32$ different shapes, where each shape is freshly sampled uniformly at random to produce $128^2$ points.

\section{Additional Results}

\subsection{Shape Space Learning}
As mentioned in \ref{ss:shape_spaces} we present additional results from the shape space learning experiment in Figure \ref{fig:supp_traintest}. We provide reconstruction results of both training and test (\ie, unseen point clouds) sets, with the random train-test split. These results are discussed in Section \ref{ss:shape_spaces}.

\section{Theory}

\subsection{Plane Reproduction using Liapunov Function}
\paragraph{}In this section we suggest an alternative, self-contained proof for the plane reproduction property of our model in the non-noisy data case, \ie,  $\gX=\set{\vx_i}_{i\in I}$ span some $d-1$ dimension hyperplane $ \gH\subset\mathbb{R}^d $ that contains the origin. 

\paragraph{}We present a simple argument that, with random initialization, the gradient flow in \eqref{e:ode_w(t)_changed} converges, with probability one, to one of the two global minima corresponding to the signed distance function to $\gH$ characterized in Theorem \ref{thm:critical}. We work in the transformed coordinate space and consider the gradient flow
\begin{equation}\label{e:ode_w(t)_changed}
\frac{d\vq}{dt}=-\nabla_\vq \ell(\vq),    
\end{equation}
with $\ell(\vq)$ as in \eqref{e:loss_linear_changed}.

\setcounter{theorem}{2}
\begin{theorem}\label{thm:linear_case}
When initializing the gradient flow in \eqref{e:ode_w(t)_changed} randomly, then with probability one the solution convergences $$\vq(t)\xrightarrow{t\too \infty}\vq^*,$$ where $\vq^*$ is one of the global minima of the loss in \eqref{e:loss_linear_changed}, \ie, $\pm\vq$. Therefore, the limit model, $f(\vx;\vq^*)$, approximates the signed distance function to $\ve_1^\perp$ (\ie, $\gH$ in the transformed coordinates). 
\end{theorem}

We prove Theorem \ref{thm:linear_case} using a certain \emph{Liapunov function} (explained shortly).
By random initialization we mean $\vq^0$ is drawn from some continuous probability distribution in $\Real^d$ (\ie, with a density function). 
Note that with probability one $\vq^0$ is not orthogonal to $\ve_1$. Let $\vv$ be one of $\pm\vq=\pm{\scriptstyle\sqrt{1-\frac{\lambda_1}{2\lambda}}}\ve_1$ from Theorem \ref{thm:critical} so that $\vv^T\vq^0>0$.

\paragraph{Liapunov function.} To show that $\vq(t)$ converges to $\vv$ we will introduce a \emph{Liapunov function}; the existence of such a function implies the desired convergence using standard stability results from the theory of dynamical systems \cite{wiggins2003introduction,teschl2012ordinary}.  Consider the domain $\Omega=\set{\vq\in \Real^d \vert \ve_1^T \vq>0}$. $h:\Omega\too\Real$ is a \emph{Liapunov function} if it is $C^1$ and satisfies the following conditions:
\begin{enumerate}
    \item \emph{Energy:} $h(\vv)=0$ and $h(\vq)>0$ for all $\vq\in\Omega\setminus\set{\vv}$.
    \item \emph{Decreasing:} $\nabla h(\vq) \cdot \frac{d\vq}{dt}(\vq)<0$ for all $\vw\in\Omega\setminus\set{\vv}$. 
    \item \emph{Bounded:} The level-sets $\set{\vq \vert h(\vq)=c}$ are bounded. 
\end{enumerate}

\begin{wrapfigure}[11]{r}{0.3\columnwidth}
     \centering
     \includegraphics[width=0.3\columnwidth]{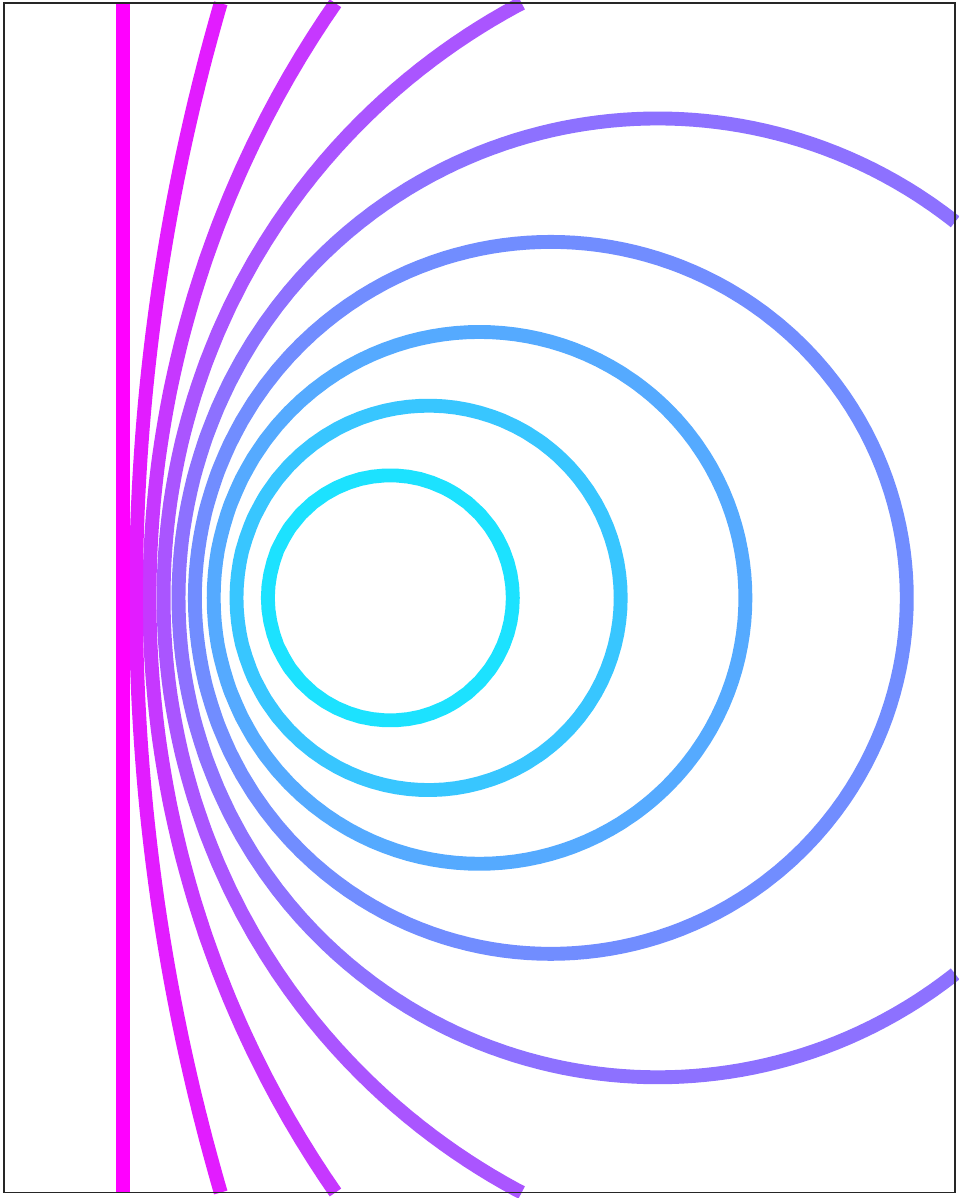}\vspace{-10pt}\caption{Level-sets of $h$.}\label{fig:liapunov}
\end{wrapfigure}
Intuitively, a Liapunov function can be imagined as a sort of an energy function (\ie, non-negative) that vanishes only at $\vv$ and that the flow defined by \eqref{e:ode_w(t)_changed} strictly decreases its value at every point, except at the fixed point $\vv$. These conditions imply that if a flow (\ie, integral curve) starting at $\vq_0\in \Omega$ stays bounded it has to converge to $\vv$. See for example Theorem 6.14 in \cite{teschl2012ordinary}. Now, consider 
\begin{equation}\label{e:liapunov}
h(\vq) = \frac{\norm{\vq-\vv}^2}{1+\norm{\vq}^2}.   \end{equation}
We will prove that $h$ is Liapunov for our problem. First it clearly satisfies the \emph{energy} condition. The \emph{bounded} condition can be seen by noting that $h(\vq)\in[0,1)$ for all $\vq\in\Omega$ and that in the quadratic equation $h(\vq)=c$ the quadratic term has the form $(1-c)\norm{\vw}^2$ and since $(1-c)>0$ the level-sets of $h$ are all finite-radius circles, see Figure \ref{fig:liapunov}. 

To prove the \emph{decreasing} property a direct computation shows that for $\vq\in \Omega$
 { \begin{equation*}
  \nabla h\cdot \frac{d\vq}{dt} = \frac{-8\vv^T \vq}{(1+\norm{\vq}^2)^2}\parr{\vq^T D \vq + \lambda\parr{\norm{\vq}^2-1}^2} \leq 0
\end{equation*}}
where in the last inequality we used the fact that $\vv^T\vq = q_1>0$, and $D$ is a positive semi-definite matrix, \ie, $\lambda_i \geq 0$, $i\in[d]$. Furthermore, if the r.h.s.~equals zero then $\vq^T D \vq=0$ and $\norm{\vq}=1$; this implies that $\vq=\vv$. Therefore for all $\vq\in\Omega\setminus \set{\vq}$ we have $\nabla h\cdot \frac{d\vq}{dt}<0$. $\qed$

\paragraph{Relation to Theorem \ref{thm:converge}.}
Although this seems as a special case of Theorem \ref{thm:converge}, note that it works for the continuous gradient flow. This is in contrast to the proof of Theorem \ref{thm:converge} that uses the result of \cite{lee2016gradient} building upon the discrete nature of gradient descent iterations. Furthermore, we believe a simple self-contained convergence proof that does not rely on previous work could be of merit. 

\begin{figure*}[t!]
    \centering
      \includegraphics[width=0.9 \textwidth]{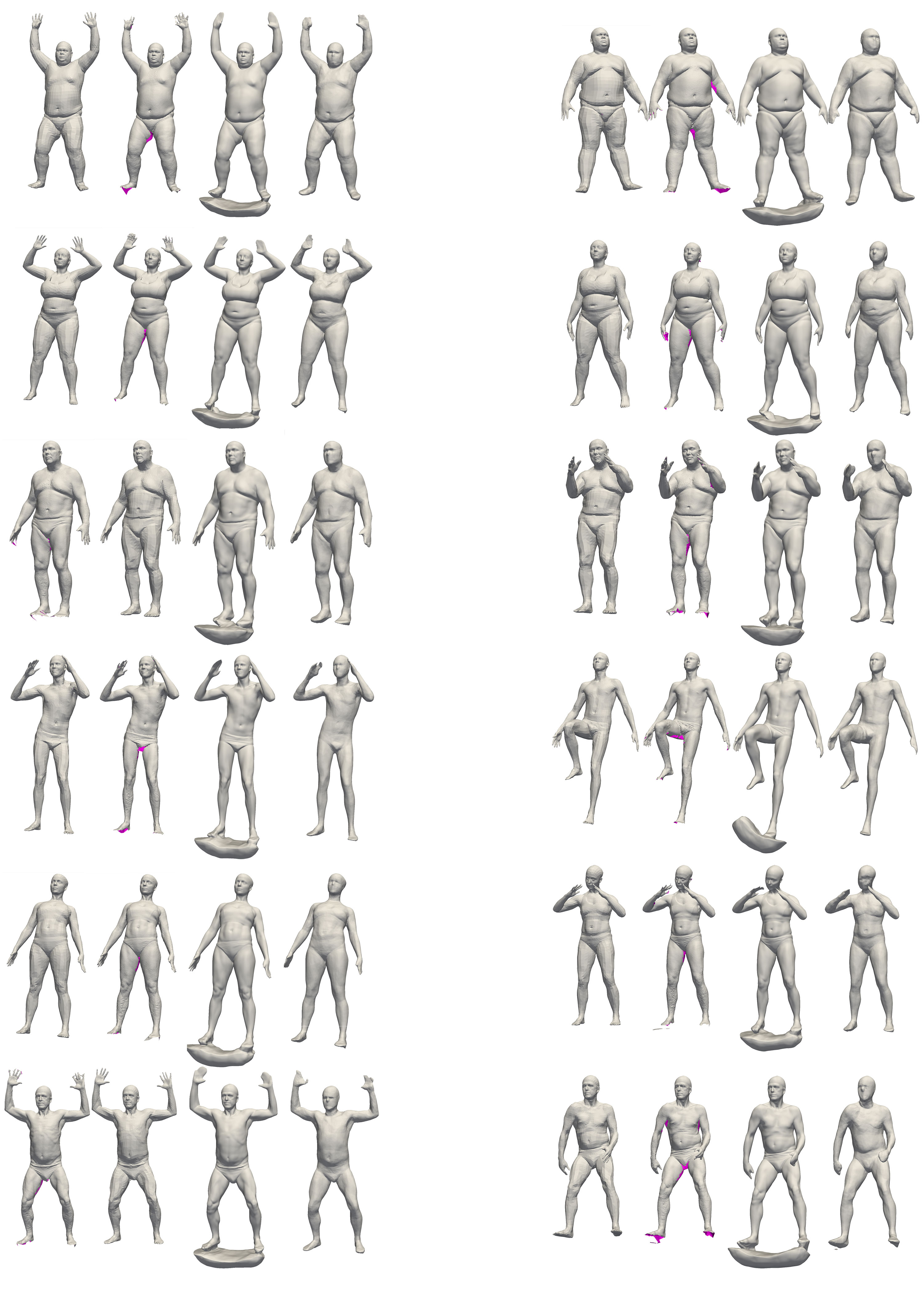}
    \caption{Additional results from D-Faust shape space experiment (see Section \ref{ss:shape_spaces} in main paper).
    Left - train results, right - test results. 
    In each row (left to right): Registration (not used), raw scan (source of input point clouds), our result, and SAL result. 
    Back-faces are colored in magenta.}
    \label{fig:supp_traintest}
    \vspace{-10pt}
\end{figure*}

\end{document}